\documentclass[11pt]{article}
\usepackage[margin=1in]{geometry}
\usepackage{natbib}
\usepackage[colorlinks,citecolor=blue,urlcolor=blue]{hyperref}

\usepackage{tablefootnote}

\usepackage{graphicx}

\usepackage{amsthm,amsmath,amsfonts,amssymb}
\usepackage{times}
\usepackage{color}
\usepackage{enumitem}
\usepackage{multirow}
\usepackage{subcaption}
\usepackage{appendix}

\usepackage{authblk}

\newtheorem{theorem}{Theorem}
\newtheorem{lemma}[theorem]{Lemma}
\newtheorem{corollary}[theorem]{Corollary}
\newtheorem{definition}[theorem]{Definition}
\newtheorem{proposition}{Proposition}
\newtheorem{remark}[theorem]{Remark}
\newtheorem{claim}[theorem]{Claim}

\usepackage[linesnumbered,boxed,norelsize]{algorithm2e}
\SetKwRepeat{Do}{do}{while}%
\SetKwInput{KwInput}{Input}
\SetKwInput{KwOutput}{Output}
\SetKwInput{KwFunction}{Function}
\SetKwInput{KwParameters}{Parameters}
\SetKwComment{Comment}{$\triangleright$\ }{}
\SetCommentSty{itshape}
\usepackage{caption}

 \usepackage{multirow}

\allowdisplaybreaks

\newtheorem{observation}{Observation}

\let\oldnl\nl
\newcommand{\nonl}{\renewcommand{\nl}{\let\nl\oldnl}}

\newcommand{\vct}{\boldsymbol }

\newcommand{\ud}{\mathrm d}

\def\cA{\mathcal{A}}

\def\cI{\mathcal{I}}

\renewcommand{\hat}{\widehat}

\renewcommand{\bar}{\overline}

\definecolor{DSgray}{cmyk}{0,1,0,0}

\newcommand{\eps}{\epsilon}

\newcommand\norm[1]{\|#1\|}

\begin{document}

\title{\Large \bf Nearly Minimax-Optimal Regret for Linearly Parameterized Bandits \thanks{Accepted for presentation at the Conference for Learning Theory (COLT) 2019. Author names listed in alphabetical order.}}

\author[1]{Yingkai Li}
\author[2]{Yining Wang}
\author[3]{Yuan Zhou}
\affil[1]{Department of Computer Science, Northwestern University}
\affil[2]{Warrington College of Business, University of Florida}
\affil[3]{Department of ISE, University of Illinois at Urbana-Champaign}

\date{}

\maketitle

{
\abstract
 
 We study the linear contextual bandit problem with finite action sets. When the problem dimension is $d$, 
 the time horizon is $T$, and there are $n \leq 2^{d/2}$ candidate actions per time period, we 
 (1) show that the minimax expected regret is $\Omega(\sqrt{dT (\log T) (\log n)})$ for every algorithm, and 
 (2) introduce a Variable-Confidence-Level (VCL) SupLinUCB algorithm whose regret matches the lower bound up to iterated logarithmic factors. 
 Our algorithmic result saves two $\sqrt{\log T}$ factors from previous analysis, and our information-theoretical lower bound also improves previous results by one $\sqrt{\log T}$ factor,
 revealing a regret scaling quite different from classical multi-armed bandits in which no logarithmic $T$ term is present in minimax regret.
 Our proof techniques include variable confidence levels and a careful analysis of layer sizes of SupLinUCB on the upper bound side, 
 and delicately constructed adversarial sequences showing the tightness of elliptical potential lemmas on the lower bound side. 
 
}

\section{Introduction}

The stochastic multi-armed bandit (MAB) problem is a sequential experiment in which sequential decisions are made 
over $T$ time periods in order to maximize the expected cumulative reward of the made decisions.
First studied by \cite{thompson1933likelihood,robbins1952some} and many more works thereafter \citep{bubeck2012regret,lai1985asymptotically,lai1987adaptive,auer2002using},
the MAB problems are one of the simplest yet most popular frameworks to study exploration--exploitation tradeoffs in sequential experiments.


In real-world applications such as advertisement selection \citep{abe2003reinforcement}, recommendation systems \citep{Li:10} and information retrieval \citep{yue2011linear},
side information is most of the time available for each possible actions.
\emph{Contextual} bandit models are thus proposed 
to incorporate such contextual information into sequential decision making.
While the study of general contextual bandit models is certainly of great interest \citep{agarwal2014taming,agarwal2012contextual,luo2017efficient},
many research efforts have also been devoted into an important special case of the contextual bandit model,
in which the mean rewards of actions are parameterized by \emph{linear} functions \citep{abe2003reinforcement,auer2002using,chu2011contextual,abbasi2011improved,abbasi2012online,dani2008stochastic,rusmevichientong2010linearly}.
We refer the readers to Sec.~\ref{sec:related} for a more detailed accounts of existing results along this direction.

In this paper, we consider the linear contextual bandit problem with finite action sets, known time horizon and oblivious action context.
We derive upper and lower bounds on the best worst-case cumulative regret any policy can achieve, that match each other except for iterated logarithmic terms
(see Table \ref{tab:results} for details and comparison with existing works).
Many new proof techniques and insights are generated, as we discuss in Sec.~\ref{sec:results}.

\subsection{Problem formulation and minimax regret}

There are $T\geq 1$ time periods, conveniently denoted as $\{1,2,\cdots,T\}$, and a fixed but unknown $d$-dimensional regression model $\theta$.
Throughout this paper we will assume the model is normalized, meaning that $\|\theta\|_2\leq 1$.
At each time period $t$, a \emph{policy} $\pi$ is presented with an \emph{action set} $\mathcal A_t=\{x_{it}\}\subseteq\{x\in\mathbb R^d: \|x\|_2\leq 1\}$, where $i$ is the index of for the candidate action in $\mathcal A_t$.
An adversary will choose the action sets $\mathcal A_1,\cdots,\mathcal A_T$ \emph{before} the policy is executed, in an arbitrary way.
The policy then chooses, based on the feedback from previous time periods $\{1,2,\cdots,t-1\}$, either deterministically or randomly an action $x_{it}\in\mathcal A_t$
and receives a reward $r_t=x_{it}^\top\theta + \eps_t$, where $\{\eps_t\}$ are independent centered sub-Gaussian random variables with variance proxy $1$, representing noise
during the reward collection procedure.
The objective is to design a good policy $\pi$ that tries to maximize its expected cumulative reward $\mathbb E\sum_{t=1}^T r_t$.

More specifically, a policy $\pi$ designed for $d$-dimensional vectors, $T$ time periods and maximum action set size $n=\max_{t\leq T}|\mathcal A_t|$  
can be parameterized as $\pi=(\pi_1,\pi_2,\cdots,\pi_T)$ such that 
\begin{equation*}
i_t = \left\{\begin{array}{ll}
\pi_1(\nu, \mathcal A_1),& t=1;\\
\pi_t(\nu, \mathcal A_1, r_1, \cdots, \mathcal A_{t-1}, r_{t-1}, \mathcal A_t),& t=2,\cdots, T,\end{array}\right.
\end{equation*}
where $\nu$ is a random quantity defined over a probability space that generates randomness in policy $\pi$.
We use $\Pi_{T,n,d}$ to denote the class of all policies defined above.

To evaluate the performance of a policy $\pi$, we consider its expected \emph{regret} $\mathbb E[R^T]$,
defined as the sum of the differences of the rewards between the policy's choosing actions and the optimal action in hindsight.
More specifically, for a policy $\pi$ and a pre-specified action sets sequence $\mathcal A_1,\cdots,\mathcal A_T$,
the expected regret is defined as
\begin{equation}
\mathbb E[R^T] = \mathbb E\left[\sum_{t=1}^T \max_{x_{it}\in\mathcal A_t}x_{it}^\top\theta - x_{i_t,t}^\top\theta\right].
\label{eq:expected-regret}
\end{equation}

Clearly, the expected regret defined in Eq.~(\ref{eq:expected-regret}) depends both on the policy $\pi$ and the environment $\theta$, $\{\mathcal A_t\}$.
Hence, a policy that has small regret for one set of environment parameters might incur large regret for other sets of environment parameters.
To provide a unified evaluation criterion, we adopt the concept of \emph{worst-case regret} and aim to find a policy that has the smallest possible worst-case regret.
More specifically, we are interested in the following defined \emph{minimax regret}
\begin{equation}
\mathfrak R(T;n,d) := \inf_{\pi\in\Pi_{T,n,d}}\sup_{\theta\in\mathbb R^d, |\mathcal A_t|\leq n} \mathbb E[R^T].
\label{eq:minimax-regret}
\end{equation}
Note that for $n=\infty$, the supremum is taken over all closed $\mathcal A_t \subseteq \{x\in\mathbb R^d: \|x\|_2\leq 1\}$ for all $t$.

The minimax framework has been increasingly popular in identifying information-theoretical limits of learning and statistics problems \citep{tsybakov2009introduction,wasserman2013all,ibragimov2013statistical}
and was applied to bandit problems as well \citep{audibert2009minimax}.

Note also that, as described in Eq.~(\ref{eq:minimax-regret}), the problem instances we are considering in this paper are \emph{oblivious} \citep{arora2012online} with \emph{finite horizons},
meaning that the regression model $\theta$ and action sets sequences $\{\mathcal A_t\}_{t=1}^T$ are chosen adversarially \emph{before} the execution of the policy $\pi$,
and the policy knows the time horizon $T$ before the first time period $t=1$.

\paragraph{Asymptotic notations.} For two sequences $\{a_n\}$ and $\{b_n\}$, we write $a_n=O(b_n)$ or $a_n\lesssim b_n$ if there exists a \emph{universal} constant $C<\infty$
such that $\limsup_{n\to\infty} |a_n|/|b_n|\leq C$.
Similarly, we write $a_n=\Omega(b_n)$ or $a_n\gtrsim b_n$ if there exists a \emph{universal} constant $c>0$ such that $\liminf_{n\to\infty} |a_n|/|b_n| \geq c$.
We write $a_n=\Theta(b_n)$ or $a_n\asymp b_n$ if both $a_n\lesssim b_n$ and $a_n\gtrsim b_n$ hold.
In asymptotic notations, we will drop base notations of logarithms and use instead $\log x$ for both $\ln x,\log_2 x$ as well as logarithms with other constant base numbers.
In non-asymptotic scenarios, however, base notations will not be dropped and $\ln x$ refers specifically to $\log_e x$.

\subsection{Related works} \label{sec:related}


The linear contextual bandit setting was introduced by Abe et al.~\cite{abe2003reinforcement}.  Auer \cite{auer2002using} and Chu et al.~\cite{chu2011contextual} proposed the SupLinRel and SupLinUCB algorithms respectively, both of which achieve  $O(\sqrt{d T} \log^{3/2} (nT))$ regret. When there are $n = \Theta(d)$ arms per round, Chu et al.~\cite{chu2011contextual} showed an $\Omega(\sqrt{dT})$ minimax regret lower bound.
A detailed account of these results are given in Table \ref{tab:results}.


Note that our problem requires that there are only finitely many candidate actions per round. When the number of candidate actions is not bounded, Dani et al.~\cite{dani2008stochastic} and Rusmevichientong et al.~\cite{rusmevichientong2010linearly} showed algorithms that achieve $O(d \sqrt{T} \log^{3/2} T)$ regret. This bound was later improved to $O(d \sqrt{T} \log T)$ by Abbasi-Yadkori et al.~\cite{abbasi2011improved}. Dani et al.~\cite{dani2008stochastic} also showed an $\Omega(d \sqrt{T})$ regret lower bound when there are $2^{\Theta(d)}$ candidate actions. 
{
Our lower bound, on the other hand, implies an $\Omega(d\sqrt{T\log T})$ lower bound for the infinite-action case
\emph{when the action space changes over time}. A detailed discussion is given in Corollary \ref{cor:infinite-action}
in Sec.~\ref{sec:infinite-action}.
}

While this paper focuses on the regret minimization task for linear contextual bandits, the pure exploration scenario also attracts much research attention in both the ordinary bandit setting (e.g.\ \cite{mannor2004sample,Karnin:13,Jamieson:14,ChenLiQiao:17}) and the linear contextual setting (e.g.\ \cite{soare2014best,tao2018best,XHS18}).

It is also worth noting that for the ordinary multi-armed bandit problem (where the $n$ arms are independent and not associated with contextual information), the MOSS algorithm \citep{audibert2009minimax} achieves $O(\sqrt{n T})$ expected regret; and the matching lower bound was proved by Auer et al.~\cite{auer1995gambling}.
{The idea of using \emph{adaptive} confidence levels in upper confidence bands has also been extensively studied
in the context of finite-armed bandit settings to remove additional logarithmic factors; 
see, for example, \citep{auer2010ucb,audibert2009minimax,lattimore2018refining,garivier2018kl}.
To our knowledge, the adaptive confidence intervals idea has \emph{not} yet been applied to contextual bandit problems
with the objective of refined regret analysis.
}

\subsection{Our results}\label{sec:results}

The main results of this paper are the following two theorems that upper and lower bound the minimax regret $\mathfrak R(T;n,d)$ for various problem parameter values.
\begin{theorem}[Upper bound]
For any $n<\infty$, the minimax regret $\mathfrak R(T;n,d)$ can be asymptotically upper bounded by $\mathrm{poly}(\log\log(nT))\cdot O(\sqrt{dT(\log T)(\log n)})$.
\label{thm:upper-finite}
\end{theorem}

\begin{theorem}[Lower bound]
For any small constant $\epsilon > 0$, and any $n,d, $ such that $n \leq 2^{d/2}$ and $T \geq d (\log_2 n)^{1+\epsilon}$,
the minimax regret $\mathfrak R(T;n,d)$ can be asymptotically lower bounded by $\Omega(1)\cdot\sqrt{dT(\log n)\log (T/d)}$.
\label{thm:lower-finite}
\end{theorem}

\begin{remark}
In Theorem \ref{thm:upper-finite}, $\mathrm{poly}(\log\log(nT))=(\log\log(nT))^\gamma$ for some constant $\gamma>0$;
in Theorem \ref{thm:lower-finite}, the $\Omega(1)$ notation hides constants that depend on $\epsilon>0$.
\end{remark}

\begin{table}[t]
\begin{center}
\captionsetup{font=normalsize,justification=centering}
   \caption{Previous results and our results on upper and lower bounds of $\mathfrak R(T;n,d)$.    \label{tab:results}}
{    \begin{tabular}{c|c|c|c} 
    \hline
       \multicolumn{2}{c|}{}    & Upper bound & Lower bound  \\  \hline
      \multirow{2}{*}{$n<\infty$}  &  \shortstack{ Previous \\ result}  &\shortstack{$O\big(\sqrt{dT}\times  \log^{3/2} (nT)\big)$  \\ \citep{auer2002using,chu2011contextual} } &\shortstack{$\Omega \big(\sqrt{dT}\big)$ \\\citep{chu2011contextual}}   \\ \cline{2-4}
          & Our result & $O\big(\sqrt{d T (\log T) (\log n)}\big)  \cdot \mathrm{poly}(\log \log (nT))$ & $\Omega\big(\sqrt{d T( \log n) \log (T/d)}\big)$ \textsuperscript{$\dagger$}   \\   \hline
    \end{tabular}}\\
  {\footnotesize \textsuperscript{$\dagger$}Under conditions $n\leq 2^{d/2}$ and $T\geq d(\log_2 n)^{1+\epsilon}$ for some constant $\epsilon>0$.}
 \end{center}
\end{table}

Comparing Theorems \ref{thm:upper-finite} and \ref{thm:lower-finite}, 
we see that the upper and lower bounds nearly match each other up to iterated logarithmic terms when $n$ (the number of actions per time period) is not too large.
While Theorems \ref{thm:upper-finite} and \ref{thm:lower-finite} technically only apply to finite $n$ cases,
we observe that the lower bound (Theorem \ref{thm:lower-finite}) extends to the $n=\infty$ case directly as it is a harder problem, and improves the previous result by \cite{dani2008stochastic}.


So far as we are aware, our Theorem~\ref{thm:lower-finite} provides the first $\sqrt{T \log T}$-style lower bound under gap-free settings in multi-armed bandit literature. Even when the degrees of freedom for unknown parameters are constants for both problems (i.e., $n = d = O(1)$), our theorem shows that linear bandits is harder than ordinary multi-armed bandits, because of the variation of arms over the time periods,  which marks a separation between the two problems.

\subsection{Techniques and insights}

On the upper bound side, we use two main techniques to remove additional logarithmic factors from previous analysis.
Our first technique is to use \emph{variable confidence levels}, by allowing the failure probability to increase as the policy progresses,
because late fails usually lead to smaller additionally incurred regret.
Our second idea to remove unnecessary logarithmic factors is to use a more careful analysis of each ``layers'' in the SupLinUCB algorithm \citep{chu2011contextual}.
Previous analysis like \citep{auer2002using,chu2011contextual} uses the total number of time periods $T$ to upper bound the sizes of each layer,
resulting in an addition $O(\sqrt{\log T})$ term as there are $\Theta(\log T)$ layers.
In our analysis, we develop a more refined theoretical control over the sizes of each layer, and show that the layer sizes have an exponentially increasing property.
With such a property we are able to remove an additional $O(\sqrt{\log T})$ term from the regret upper bounds.

On the lower bound side, we consider a carefully designed sequence $\{z_t\}$ 
(see the proof of Lemma \ref{lem_potential_low} for details)
which shows the tightness of the \emph{elliptical potential lemma}, a key technical step in the proof of all previous analysis of linearly parameterized bandits
and their variants \citep{abbasi2011improved,dani2008stochastic,auer2002using,chu2011contextual,li2017provable,filippi2010parametric,rusmevichientong2010linearly}.
The constructed sequence $\{z_t\}$ not only shows the tightness of existing analysis, but also motivated our construction of adversarial problem instances
that lower bound regret of general bandit algorithms.

\section{Upper bounds}\label{sec:ub}

We propose \emph{Variable-Confidence-Level (VCL) SupLinUCB},
a variant of the SupLinUCB algorithm \citep{auer2002using,chu2011contextual}
that uses variable confidence levels in the construction of confidence intervals at different stages of the algorithm.
We then derive an upper regret bound that is almost tight in terms of dependency on the problem parameters, especially
the time horizon parameter $T$.

\subsection{The VCL-SupLinUCB algorithm}

Algorithm \ref{alg:suplinucb} describes our proposed VCL-SupLinUCB algorithm. 
The algorithm is a variant of the SupLinUCB algorithm proposed in \citep{auer2002using,chu2011contextual},
with variable confidence levels at different time periods.

{

\subsubsection{High-level intuitions and structures of VCL-SupLinUCB}

Our proposed VCL-SupLinUCB algorithm is based on the classical SupLinUCB algorithm,
which uses the idea of ``layered data partitioning'' to resolve delicate data dependency problems in sequential decision making.
In this sub-section we give a high-level description of the core ideas behind SupLinUCB.

Recall that in linear contextual bandit, at each time $t\in\{1,2,\cdots,T\}$ a set of $n$ contextual/feature vectors $\{x_{it}\}_{i=1}^n$
are given, and the algorithm is tasked with selecting one vector $i_t\in[n]$ in the hope of maximizing the expected linear payoff $\langle x_{i_t,t},\theta\rangle$.
The most natural way is to obtain an ordinary-least squares (OLS) estimate $\hat\theta_t$ on all payoff data collected prior to time $t$,
and then select $i_t\in[n]$ that maximizes $\langle x_{i_t,t},\hat\theta_t\rangle$ plus upper confidence bands.
One major disadvantage of this approach, however, is the implicit statistical correlation hidden in the OLS estimate $\hat\theta_t$,
preventing rigorous analysis attaining tight regret bounds.
Indeed, to our knowledge the best regret upper bounds established for such procedures are $O(d\sqrt{T\log^2 T})$ \citep{abbasi2012online},
which is not tight in either $O(\log T)$ or $O(d)$ terms.

The work of Auer~\cite{auer2002using}, with follow-ups in \citep{chu2011contextual,li2017provable}, uses the idea of ``layered data partitioning''
to develop a significantly more complex algorithm (also known as SupLinUCB) to resolve the data dependency problem.
Instead of using all payoff data prior to time $t$ for an OLS estimate, the time periods prior to $t$ are divided into disjoint ``layers'' $\zeta\in\{0,1,\cdots,\zeta_0\}$,
or more specifically $\{\mathcal X_{\zeta,t}\}_{\zeta=0}^{\zeta_0}$ such that $\dot\bigcup_{\zeta=0}^{\zeta_0}\mathcal X_{\zeta,t} = \{1,2,\cdots,t-1\}$.
For larger values of $\zeta$, the partitioned data subset $\mathcal X_{\zeta,t}$ is likely to contain more data points prior to $t$,
hence the resulting OLS estimate on data from $\mathcal X_{\zeta,t}$ are likely to be more accurate.

At time $t$, the layers $\zeta=0,1,\cdots,\zeta_0$ are visited sequentially and within each layer an OLS estimate $\hat\theta_{\zeta,t}$
is calculated on data collected solely during periods in $\mathcal X_{\zeta,t}$.
More specifically, the algorithm starts from $\zeta=0$ (corresponding to the widest confidence bands)
and increases $\zeta$, while in the mean time knocking out all sub-optimal actions using OLS estimates from $\mathcal X_{\zeta,t}$.
The procedure is carefully designed so that the ``stopping layer'' $\zeta_t$ produced at time $t$ only depends on $\{\mathcal X_{\zeta,t}\}_{\zeta\leq\zeta_t}$,
which decouples the statistical correlation in OLS estimates.
A more careful and rigorous statement of this statistical decoupling property is given in Proposition \ref{prop:layer-independence}.

\subsubsection{Notations}

In our pseudo-code description of VCL-SupLinUCB (Algorithm \ref{alg:suplinucb})
there are many notations to help define and clarify this subtle procedure.
To help with readability and checking with the technical analysis, 
we provide a tabled description and summary of the notations used in our algorithm description and analysis,
in Table \ref{tab:notation}.

}

\begin{table}[t]
\centering
\caption{Summary of notations used in Algorithm \ref{alg:suplinucb}.}
\begin{tabular}{lll}
\hline
Symbol& Range& Description\\
\hline
$t$& $\{1,2,\cdots,T\}$& the current time period\\
$i$& $\{1,2,\cdots,n\}$& the index of context/feature vectors\\
$x_{i,t}$& $\mathbb R^d$& the $i$th context/feature vector at time $t$ \\
$i_t$& $\{1,2,\cdots,n\}$& the action chosen at time $t$\\
$r_t$& $\mathbb R$& the payoff received at time $t$\\
$\zeta$& $\{0,1,\cdots,\zeta_0\}$& the index of layers in data partitioning\\
$\zeta_t$& $\{0,1,\cdots,\zeta_0\}$& the layer index at which algorithm stops at time $t$ \\
$\mathcal X_{\zeta,t}$& $\subseteq \{1,2,\cdots,t\}$& time periods prior to or at $t$ belonging to partitioned data layer $\zeta$\\
$\hat\theta_{\zeta,t}$& $\mathbb R^d$& OLS estimates obtained on payoff data for periods in $\mathcal X_{\zeta,t-1}$\\
$\omega_{\zeta,t}^i,\alpha_{\zeta,t}^i,\varpi_{\zeta,t}^i$& $\mathbb R^*$& confidence band related quantities for context $i$, time $t$ and layer $\zeta$\\
$\mathcal N_{\zeta,t}$& $\subseteq \{1,\cdots,n\}$& subset of context/feature vectors active at layer $\zeta$\\
$\lambda_{\zeta,t-1},\Lambda_{\zeta,t-1}$& $\mathbb R^d$, $\mathbb R^{d\times d}$& sufficient statistics for computing $\hat\theta_{\zeta,t}$\\
\hline
\end{tabular}
\label{tab:notation}
\end{table}

\begin{algorithm}[t]
\caption{The VCL-SupLinUCB algorithm}
\label{alg:suplinucb}

 Parameters: $\zeta_0=\lceil\log_2(\sqrt{T/d})\rceil$;\\
 Initialization: $\mathcal X_{\zeta,0}=\emptyset$, $\Lambda_{\zeta,0}=I_{d\times d}$, $\lambda_{\zeta,0}=\vec{0}$ for all $\zeta=0,1,\cdots,\zeta_0$;\\
\For{$t=1,2,\cdots,T$}{

	 Observe $\{x_{it}\}$ for $i=1,2,\cdots,n$;\\
	 Set $\zeta=0$ and $\mathcal N_{\zeta,t}=\{1,2,\cdots,n\}$;\\
	\While{a choice $i_t$ has yet to be made}{
	Compute $\hat\theta_{\zeta,t}=\Lambda_{\zeta,t-1}^{-1}\lambda_{\zeta,t-1}$;\\
	
	 For every $i\in\mathcal N_{\zeta,t}$, compute $\omega_{\zeta,t}^i = \sqrt{x_{it}^\top\Lambda_{\zeta,t-1}^{-1}x_{it}}$,
	 $\alpha_{\zeta,t}^i = 1+\max\{1, \sqrt{\ln[T(\omega_{\zeta,t}^i)^2/d]}\} \sqrt{2\ln(n\zeta_0)}$ and $\varpi_{\zeta,t}^i = \alpha_{\zeta,t}^i\omega_{\zeta,t}^i$;\\
	 \uIf{$\zeta= \zeta_0$}{

	 	 Select any $i_t \in \mathcal{N}_{\zeta, t}$, and set $\zeta_t=\zeta$\;
	 }
	 \uElseIf{$\varpi_{\zeta,t}^i\leq 2^{-\zeta}$ for all $i\in\mathcal N_{\zeta,t}$  \label{line:alg-1-l10}}{
	 

	 	 Update $\mathcal N_{\zeta+1,t} = \{i\in\mathcal N_{\zeta,t}: x_{it}^\top\hat\theta_{\zeta,t}\geq\max_{j\in\mathcal N_{\zeta,t}}x_{jt}^\top\hat\theta_{\zeta,t}-2^{1-\zeta}\}$, $\zeta\gets\zeta + 1$\; \label{line:alg-1-l11}
	 }
	 \Else{
	
	 	Select any $i_t\in\mathcal N_{\zeta,t}$ such that $\varpi_{\zeta,t}^{i_t}\geq 2^{-\zeta}$, and set $\zeta_t=\zeta$;
	 }
	}
	Play action $i_t$ and observe feedback $r_t=x_{i_t,t}^\top\theta+\varepsilon_t$;\\
	 Update: $\mathcal X_{\zeta,t}=\mathcal X_{\zeta,t-1}\cup\{t\}$, $\Lambda_{\zeta,t}=\Lambda_{\zeta,t-1}+x_{i_t,t}x_{i_t,t}^\top$, $\lambda_{\zeta,t}=\lambda_{\zeta,t-1} + r_tx_{i_t,t}$
	 for $\zeta=\zeta_t$, 
	 and $\mathcal X_{\zeta,t}=\mathcal X_{\zeta,t-1}$, $\Lambda_{\zeta,t}=\Lambda_{\zeta,t-1}$, $\lambda_{\zeta,t}=\lambda_{\zeta,t-1}$ for $\zeta\neq\zeta_t$;
}
\end{algorithm}

\subsection{Tight regret analysis}

In this section we sketch our regret analysis of Algorithm \ref{alg:suplinucb} that gives rises to almost tight $T$ dependency.
To shed lights on the novelty of our analysis, we first review existing results from \citep{chu2011contextual}
on regret upper bounds of the traditional SupLinUCB algorithm:
\begin{theorem}[\cite{chu2011contextual}]
The expected cumulative regret of the classical SupLinUCB algorithm can be upper bounded by
$
O(\sqrt{dT\log^3(nT)}).
$
\label{thm:auerchu}
\end{theorem}

It is immediately noted that the regret upper bound in Theorem \ref{thm:auerchu} has three $O(\sqrt{\log T})$ terms.
Digging into the analysis of \citep{auer2002using,chu2011contextual} we are able to pinpoint the sources of each of the $O(\sqrt{\log T})$ terms:
\begin{enumerate}
\item One $O(\sqrt{\log T})$ term arises from a union bound over all $T$ time periods;
\item One $O(\sqrt{\log T})$ term arises from the \emph{elliptical potential lemma} bounding the summation of squared confidence interval lengths;
\item One $O(\sqrt{\log T})$ term arises from the $O(\log T)$ levels of $\zeta\in\{0,1,\cdots,\zeta_0\}$.
\end{enumerate}

In this section, we will focus primarily on our techniques to remove the first and the third $O(\sqrt{\log T})$ term, while
in the next section we prove a lower bound showing that the $O(\sqrt{\log T})$ term arising from the second source cannot be eliminated for any algorithm.


As a first step of our analysis, we revisit the crucial property of statistical independence across resolution levels $\zeta$,
a core property used in prior analysis of SupLinUCB type algorithms \citep{auer2002using,chu2011contextual,li2017provable}. The proof of Proposition~\ref{prop:layer-independence} is essentially the same as that of Lemma 14 in \citep{auer2002using}.
\begin{proposition}
For any $\zeta$ and $t$, 
conditioned on $\{\mathcal X_{\zeta',t}\}_{\zeta'\leq\zeta}, \{\Lambda_{\zeta',t}\}_{\zeta'\leq\zeta}$ and $\{\lambda_{\zeta',t}\}_{\zeta'<\zeta}$,
the random variables $\{\varepsilon_{t'}\}_{t'\in\mathcal X_{\zeta,t}}$ are independent, centered sub-Gaussian random variables
with variance proxy 1.
\label{prop:layer-independence}
\end{proposition}

\subsubsection{Removing the first $O(\sqrt{\log T})$ term}

To remove the first $O(\sqrt{\log T})$ term arising from a union bound over all $T$ time periods,
our main idea is to use \emph{variable} confidence levels depending on the (square root of the) quadratic form $\omega_{\zeta,t}^i=\sqrt{x_{it}^\top\Lambda_{\zeta,t}^{-1}x_{it}}$,
instead of constant confidence levels $1/\mathrm{poly}(T)$ used in \citep{auer2002using,chu2011contextual}.
The following lemma gives an upper bound on the regret of VCL-SupLinUCB:
\begin{lemma}
{
The sequence of actions $\{i_t\}_{t=1}^T$ produced by Algorithm \ref{alg:suplinucb} satisfies
\begin{align}
\mathbb E[R^T] 
&\lesssim \sqrt{dT} + \mathbb E \left[\sum_{t=1}^T \alpha_{\zeta_t,t}^{i_t}\cdot \omega_{\zeta_t,t}^{i_t}\right]\label{eq:vcl-regret-1}\\
&\lesssim \sqrt{dT} + \sqrt{\log (n\log T)}\cdot \mathbb E\left[\sum_{t=1}^T \sqrt{\max\{1,\log[T(\omega_{\zeta_t,t}^{i_t})^2/d]\}}\cdot \omega_{\zeta_t,t}^{i_t}\right],\nonumber
\end{align}
where $\zeta_t\in\{0,1,\cdots,\zeta_0\}$ is the resolution level at time period $t$ and $\alpha_{\zeta_t,t}^{i_t},\omega_{\zeta_t,t}^{i_t}$ are defined in Algorithm \ref{alg:suplinucb}.
}
\label{lem:vcl-regret-1}
\end{lemma}

Compared to similar lemmas in existing analytical framework \citep{auer2002using,chu2011contextual}, 
the major improvement is the reduction from $\log T$ to $\log[T(\omega_{\zeta_t,t}^{i_t})^2/d]$ in the multiplier before 
the main confidence interval length term $\omega_{\zeta_t,t}^{i_t}$,
meaning that when the $\{\omega_{\zeta,t}^{i_t}\}$ shrink as more observations are collected,
the overall confidence interval length also decreases.
This helps reduce the $\log T$ term, which eventually disappears when $\omega_{\zeta,t}^{i_t}$ is sufficiently small.

To state our proof of Lemma \ref{lem:vcl-regret-1} we define some notations
and also present an intermediate lemma.
For any $\zeta,t$ such that $\zeta \leq \zeta_t$, 
define 
$
\overline m_{\zeta,t} := \max_{i\in\mathcal N_{\zeta,t}} x_{it}^\top\theta$ and
$\underline m_{\zeta,t} := \min_{i\in\mathcal N_{\zeta,t}}x_{it}^\top\theta
$
as the largest and smallest mean reward for actions within action subset $\mathcal N_{\zeta,t}$.
For convenience, we also define $\overline m_{\zeta,t} := \overline m_{\zeta_t,t}$ and $\underline m_{\zeta,t}:=\underline m_{\zeta_t,t}$ for all $\zeta > \zeta_t$.
The following lemma is central to our proof of Lemma \ref{lem:vcl-regret-1}:
{
\begin{lemma}
For all $t$ and $\zeta=0,1,\cdots,\zeta_0$, it holds that
\begin{align}
& \mathbb E\left[\overline m_{\zeta,t}-\overline m_{\zeta+1,t}\right] \leq 2\sqrt{2 \pi d}/(\zeta_0\sqrt{T});\label{eq:mt-1}\\
& \mathbb{E}\left[ \max\{\overline m_{\zeta,t} - \underline m_{\zeta,t} - 2^{3-\zeta}, 0\} \cdot \vct 1\{\zeta \leq \zeta_t\}\right]  \leq  \sqrt{2 \pi d}/(\zeta_0\sqrt{T}).\label{eq:mt-2}
\end{align}
\label{lem:mt}
\end{lemma}
}

At a higher level, Eq.~(\ref{eq:mt-1}) states that by reducing the candidate set from $\mathcal N_{\zeta,t}$ to $\mathcal N_{\zeta+1,t}$,
the action corresponding to large rewards is preserved (up to an error term of $2\sqrt{2\pi d}/(\zeta_0\sqrt{T})$);
Eq.~(\ref{eq:mt-2}) further gives an exponentially decreasing upper bound on the differences between the best and the worst actions within $\mathcal N_{\zeta,t}$, 
corroborating the intuition that as $\zeta$ increases and we go to more refined levels, the action set $\mathcal N_{\zeta,t}$ should ``zoom in'' onto the actions with the best potential rewards.
\begin{proof}[Proof of Lemma \ref{lem:mt}.]
We prove Eqs.~(\ref{eq:mt-1}) and (\ref{eq:mt-2}) separately.

\underline{Proof of Eq.~(\ref{eq:mt-1}).} Note that $\overline m_{\zeta,t}=\overline m_{\zeta+1,t}$ when $\zeta \geq \zeta_t$ and we therefore only need to prove the inequality conditioned on $\zeta < \zeta_t$.
That is, we only need to upper bound $\mathbb E[(\overline m_{\zeta,t}-\overline m_{\zeta+1,t})~|~\zeta<\zeta_t]$,
because
\begin{align*}
\mathbb E[\overline m_{\zeta,t}-\overline m_{\zeta+1,t}]
&= \mathbb E[(\overline m_{\zeta,t}-\overline m_{\zeta+1,t})~|~ \zeta\geq\zeta_t]\Pr[\zeta\geq\zeta_t] + \mathbb E[(\overline m_{\zeta,t}-\overline m_{\zeta+1,t})~|~\zeta<\zeta_t]\Pr[\zeta<\zeta_t]\\
&= \mathbb E[(\overline m_{\zeta,t}-\overline m_{\zeta+1,t})~|~\zeta<\zeta_t]\Pr[\zeta<\zeta_t]\\
&\leq \mathbb E[(\overline m_{\zeta,t}-\overline m_{\zeta+1,t})~|~\zeta<\zeta_t].
\end{align*}

Define $I^* := \arg\max_{i\in\mathcal N_{\zeta,t}}\{x_{it}^\top\theta\}$ and $J^* := \arg\max_{i\in\mathcal N_{\zeta,t}}\{x_{it}^\top\hat\theta_{\zeta,t} \}$,
{representing the ``optimal'' action in $\mathcal N_{\zeta,t}$ under the true model $\theta$ and the estimated model $\hat\theta_{\zeta,t}$, respectively.}
If $I^*\in\mathcal N_{\zeta+1,t}$, then $\overline m_{\zeta,t}=\overline m_{\zeta+1,t}$ because $\mathcal N_{\zeta+1,t}\subseteq\mathcal N_{\zeta,t}$.
Therefore, we only need to consider the case of $I^*\notin\mathcal N_{\zeta+1,t}$ for the sake of proving Eq.~(\ref{eq:mt-1}).

Note that $J^*\in\mathcal N_{\zeta+1,t}$ always holds, because $J^*$ maximizes $x_{it}^\top\hat\theta_{\zeta,t}$
in $\mathcal N_{\zeta,t}$,
and by our algorithm the maximum of $x_{it}^\top\hat\theta_{\zeta,t}$, $i\in\mathcal N_{\zeta,t}$ will never be removed
unless a decision on $i_t$ can already be made.
We then have with probability one that
{
\begin{align}
\overline m_{\zeta,t} -\overline m_{\zeta+1, t} 
&= \vct 1\{I^*\notin\mathcal N_{\zeta+1,t}\}\cdot (x_{I^*,t}^\top\theta - \max_{j\in\mathcal N_{\zeta+1,t}}x_{jt}^\top \theta)\nonumber\\
&\leq \vct 1\{I^*\notin\mathcal N_{\zeta+1,t}\}\cdot (x_{I^*,t}-x_{J^*,t})^\top\theta,
\label{eq:mm1-eq1}
\end{align}
where the second inequality holds because $J^*\in\mathcal N_{\zeta+1,t}$ and therefore $x_{J^*,t}^\top\theta \leq \max_{j\in\mathcal N_{\zeta+1,t}}x_{jt}^\top\theta$.
}

For any $\zeta,t$ and $i\in\mathcal N_{\zeta,t}$, define
$
\mathcal E_{\zeta,t}^i := \{|x_{it}^\top(\hat\theta_{\zeta,t}-\theta)| \leq \varpi_{\zeta,t}^i\}
$
as the success event in which the estimation error of $x_{it}^\top\hat\theta_{\zeta,t}$ for $x_{it}^\top\theta$ is within the confidence interval $\varpi_{\zeta,t}^i$.
By definition, with probability one it holds that
\begin{align}
x_{I^*,t}^\top\theta &\leq x_{I^*,t}^\top\hat\theta_{\zeta,t} + \varpi_{\zeta,t}^{I^*} + \vct 1\{\neg\mathcal E_{\zeta,t}^{I^*}\}\cdot \big|x_{I^*,t}^\top(\hat\theta_{\zeta,t}-\theta)\big|;\label{eq:xistar-expansion}\\
x_{J^*,t}^\top\theta &\geq x_{J^*,t}^\top\hat\theta_{\zeta,t} - \varpi_{\zeta,t}^{J^*} - \vct 1\{\neg\mathcal E_{\zeta,t}^{J^*}\}\cdot \big|x_{J^*,t}^\top(\hat\theta_{\zeta,t}-\theta)\big|.
\label{eq:xjstar-expansion}
\end{align}
Also, conditioned on the event $I^*\notin\mathcal N_{\zeta+1,t}$, the procedure of Algorithm \ref{alg:suplinucb} implies 
\begin{equation}
x_{I^*,t}^\top\hat\theta_{\zeta,t} < x_{J^*,t}^\top\hat\theta_{\zeta,t} - 2^{1-\zeta}.
\label{eq:xijstar-side}
\end{equation}
Subtracting Eq.~(\ref{eq:xjstar-expansion}) from Eq.~(\ref{eq:xistar-expansion}) and considering Eq.~(\ref{eq:xijstar-side}), we have
{
\begin{align}
(x_{I^*,t}-x_{J^*,t})^\top\theta
&\leq (x_{I^*,t}^\top\hat\theta_{\zeta,t}-x_{J^*,t}^\top\hat\theta_{\zeta,t}) + \varpi_{\zeta,t}^{I^*} + \varpi_{\zeta,t}^{J^*} + \sum_{i\in\{I^*,J^*\}}\vct 1\{\neg\mathcal E_{\zeta,t}^{i}\}\cdot \big|x_{i,t}^\top(\hat\theta_{\zeta,t}-\theta)\big| \nonumber\\
&\leq \varpi_{\zeta,t}^{I^*} + \varpi_{\zeta,t}^{J^*} - 2^{1-\zeta} + \sum_{i\in\{I^*,J^*\}}\vct 1\{\neg\mathcal E_{\zeta,t}^{i}\}\cdot \big|x_{i,t}^\top(\hat\theta_{\zeta,t}-\theta)\big|\nonumber\\
&\leq \sum_{i\in\{I^*,J^*\}}\vct 1\{\neg\mathcal E_{\zeta,t}^{i}\}\cdot \big|x_{i,t}^\top(\hat\theta_{\zeta,t}-\theta)\big|,
\label{eq:xijstar-main}
\end{align}
}
where the last inequality holds because $\varpi_{\zeta,t}^i\leq 2^{-\zeta}$ for all $i\in\mathcal N_{\zeta,t}$, if the algorithm is executed to resolution level $\zeta+1$.
Combining Eqs.~(\ref{eq:mm1-eq1},\ref{eq:xijstar-main}) and taking expectations, we obtain
\begin{align}
&\mathbb E\left[\overline m_{\zeta,t}-\overline m_{\zeta+1,t} ~|~ \zeta < \zeta_t\right] \nonumber \\
\leq & \mathbb E\left[\vct 1\{I^*\notin\mathcal N_{\zeta+1,t}\}\cdot \left(\vct 1\{\neg\mathcal E_{\zeta,t}^{I^*}\} \big|x_{I^*,t}^\top(\hat\theta_{\zeta,t}-\theta)\big|+\vct 1\{\neg\mathcal E_{\zeta,t}^{J^*}\} \big|x_{J^*,t}^\top(\hat\theta_{\zeta,t}-\theta)\big|\right) ~|~ \zeta < \zeta_t\right]\nonumber\\
\leq & \mathbb E\left[\vct 1\{\neg\mathcal E_{\zeta,t}^{I^*}\} \big|x_{I^*,t}^\top(\hat\theta_{\zeta,t}-\theta)\big| ~|~ \zeta < \zeta_t \right]+\mathbb E\left[\vct 1\{\neg\mathcal E_{\zeta,t}^{J^*}\} \big|x_{J^*,t}^\top(\hat\theta_{\zeta,t}-\theta)\big| ~|~ \zeta < \zeta_t\right].
\label{eq:mt1-eq2}
\end{align}

The following lemma gives an upper bound on the two terms in Eq.~(\ref{eq:mt1-eq2}):
\begin{lemma}
For any $\zeta,t$ and $i$, we have that 
$$
\mathbb E\big[\vct 1\{\neg\mathcal E_{\zeta,t}^i\}\cdot\big|x_{it}^\top(\hat\theta_{\zeta,t}-\theta)\big| ~|~ \zeta< \zeta_t,i\in\mathcal N_{\zeta,t}\big] \leq \sqrt{2 \pi d}/(n\zeta_0\sqrt{T}).
$$
\label{lem:expected-tail}
\end{lemma}

Lemma \ref{lem:expected-tail} can be proved by using the statistical independence between $\{\varepsilon_{t'}\}_{t'\in\mathcal N_{\zeta,t}}$ and $\{x_{i_{t'},t'}\}_{t'\in\mathcal N_{\zeta,t}}$  (Proposition \ref{prop:layer-independence}), 
and integration of least-squares estimation errors.
As its proof is technical but rather routine, we defer it to Appendix~\ref{app:proofs-ub}.
{ With Lemma \ref{lem:expected-tail} and Eq.~(\ref{eq:mt1-eq2}), we have that
\begin{align*}
\mathbb E\left[\overline m_{\zeta,t}-\overline m_{\zeta+1,t} ~|~ \zeta < \zeta_t\right] & \leq 2 \times  \sum_{i=1}^n  \mathbb E[\vct 1\{\neg\mathcal E_{\zeta,t}^i\}\cdot |x_{it}^\top(\hat\theta_{\zeta,t}-\theta)| \cdot \vct 1\{i \in \mathcal{N}_{\zeta, t}\} ~|~ \zeta <\zeta_t]\\
& \leq 2 \times  \sum_{i=1}^n  \mathbb E[\vct 1\{\neg\mathcal E_{\zeta,t}^i\}\cdot |x_{it}^\top(\hat\theta_{\zeta,t}-\theta)|  \vct ~|~ \zeta <\zeta_t \text{~and~} i \in \mathcal{N}_{\zeta, t}]\\
& \leq 2n \times \sqrt{2 \pi d}/(n\zeta_0\sqrt{T})  =  2\sqrt{2 \pi d}/(\zeta_0\sqrt{T}) ,
\end{align*}
concluding the proof of  Eq.~(\ref{eq:mt-1}). }

\underline{Proof of Eq.~(\ref{eq:mt-2}).}
If $\zeta=0$ then $2^{3-\zeta}=8>\overline m_{\zeta,t}-\underline m_{\zeta,t}$ almost surely, because $|x_{it}^\top\theta|\leq 1$ for all $i$ and $t$.
This immediately implies Eq.~(\ref{eq:mt-2}) because $\overline m_{\zeta,t}-\underline m_{\zeta,t}-2^{3-\zeta} < 0$.
Hence in the proof we only need to consider the case of $\zeta>0$.

Now condition on the event $\zeta \leq \zeta_t$. Define $I^* := \arg\max_{i\in\mathcal N_{\zeta,t}}\{x_{it}^\top\theta\}$ and $J^* := \arg\min_{j\in\mathcal N_{\zeta,t}}\{x_{jt}^\top\theta\}$,
{representing the context vectors in $\mathcal N_{\zeta,t}$ with the largest and smallest inner products under the true model $\theta$, respectively.}
By definition, $\overline m_{\zeta,t}-\underline m_{\zeta,t}=(x_{I^*,t}-x_{J^*,t})^\top\theta$.
Similar to Eqs.~(\ref{eq:xistar-expansion},\ref{eq:xjstar-expansion}), we can establish that, with probability one,
\begin{align}
x_{I^*,t}^\top\theta &\leq x_{I^*,t}^\top\hat\theta_{\zeta-1,t} + \varpi_{\zeta-1,t}^{I^*} + \vct 1\{\neg\mathcal E_{\zeta-1,t}^{I^*}\}\cdot \big|x_{I^*,t}^\top(\hat\theta_{\zeta-1,t}-\theta)\big|;\label{eq:xistar-expansion-2}\\
x_{J^*,t}^\top\theta &\geq x_{J^*,t}^\top\hat\theta_{\zeta-1,t} - \varpi_{\zeta-1,t}^{J^*} - \vct 1\{\neg\mathcal E_{\zeta-1,t}^{J^*}\}\cdot \big|x_{J^*,t}^\top(\hat\theta_{\zeta-1,t}-\theta)\big|.
\label{eq:xjstar-expansion-2}
\end{align}
In addition, because both $I^*$ and $J^*$ belong to $\mathcal N_{\zeta,t}$, Line~\ref{line:alg-1-l11} of Algorithm \ref{alg:suplinucb} implies that 
\begin{equation}
x_{J^*,t}^\top\hat\theta_{\zeta-1,t} \geq x_{I^*,t}^\top\hat\theta_{\zeta-1,t} - 2^{1-(\zeta-1)} \geq x_{I^*,t}^\top\hat\theta_{\zeta-1,t}-2^{2-\zeta}.
\label{eq:xijstar-side-2}
\end{equation}
Subtracting Eq.~(\ref{eq:xistar-expansion-2}) from Eq.~(\ref{eq:xjstar-expansion-2}) and applying Eq.~(\ref{eq:xijstar-side-2}), we have
\begin{align*}
x_{I^*,t}^\top\theta - x_{J^*,t}^\top\theta  & \leq   x_{I^*,t}^\top\hat\theta_{\zeta-1,t}  - x_{J^*,t}^\top\hat\theta_{\zeta-1,t}  + \varpi_{\zeta-1,t}^{I^*} + \varpi_{\zeta-1,t}^{J^*} + \sum_{i \in \mathcal{N}_{\zeta-1, t}} [\vct 1\{\neg\mathcal E_{\zeta-1,t}^i\}\cdot |x_{it}^\top(\hat\theta_{\zeta-1,t}-\theta)|] \\
& \leq  2^{2-\zeta} + \varpi_{\zeta-1,t}^{I^*} + \varpi_{\zeta-1,t}^{J^*} +  \sum_{i \in \mathcal{N}_{\zeta-1, t}} [\vct 1\{\neg\mathcal E_{\zeta-1,t}^i\}\cdot |x_{it}^\top(\hat\theta_{\zeta-1,t}-\theta)|]\\
& \leq 2^{3-\zeta}+ \sum_{i \in \mathcal{N}_{\zeta-1, t}} [\vct 1\{\neg\mathcal E_{\zeta-1,t}^i\}\cdot |x_{it}^\top(\hat\theta_{\zeta-1,t}-\theta)|],
\end{align*}
where the last inequality is because of Line~\ref{line:alg-1-l10} of Algorithm \ref{alg:suplinucb},
such that $\varpi_{\zeta-1,t}^{I^*},\varpi_{\zeta-1,t}^{J^*}\leq 2^{1-\zeta}$. 

{
In total, we have 
\begin{align*}
&\mathbb{E}\left[ \max\{x_{I^*,t}^\top\theta - x_{J^*,t}^\top\theta - 2^{3-\zeta}, 0\} \Big| \zeta-1 
< \zeta_t \right] \\
\leq~& \mathbb E \left[\max\big\{0, \sum_{i \in \mathcal{N}_{\zeta-1, t}} [\vct 1\{\neg\mathcal E_{\zeta-1,t}^i\}\cdot |x_{it}^\top(\hat\theta_{\zeta-1,t}-\theta)| ] \big\} \Bigg| \zeta-1 < \zeta_t\right]\\
=~& \mathbb E \left[\sum_{i \in \mathcal{N}_{\zeta-1, t}} [\vct 1\{\neg\mathcal E_{\zeta-1,t}^i\}\cdot |x_{it}^\top(\hat\theta_{\zeta-1,t}-\theta)| ] \Bigg| \zeta-1 < \zeta_t\right]\\
=~& \mathbb E \left[\sum_{i=1}^n [\vct 1\{\neg\mathcal E_{\zeta-1,t}^i\}\cdot |x_{it}^\top(\hat\theta_{\zeta-1,t}-\theta)| ] \vct 1\{i \in \mathcal{N}_{\zeta-1, t}\} \Bigg| \zeta-1 < \zeta_t\right]\\
\leq~&  \sum_{i=1}^n \mathbb E \left[ \vct 1\{\neg\mathcal E_{\zeta-1,t}^i\}\cdot |x_{it}^\top(\hat\theta_{\zeta-1,t}-\theta)|  \Big| \zeta-1 < \zeta_t \text{~and~} i \in \mathcal{N}_{\zeta-1, t}\right]\\
 \leq~&   n \times \frac{\sqrt{2\pi d}}{n\zeta_0\sqrt{T}} = \sqrt{2\pi d}/(\zeta_0\sqrt{T})
 \end{align*}
where the last inequality is due to Lemma \ref{lem:expected-tail}. We then have that
\[
\mathbb{E}\left[ \max\{x_{I^*,t}^\top\theta - x_{J^*,t}^\top\theta - 2^{3-\zeta}, 0\} \cdot \vct 1\{ \zeta\leq \zeta_t\} \right] \leq \sqrt{2\pi d}/(\zeta_0\sqrt{T}) .
\]
 Eq.~(\ref{eq:mt-2}) is then proved. 
 }
\end{proof}

We are now ready to prove Lemma \ref{lem:vcl-regret-1}.

\begin{proof}[Proof of Lemma \ref{lem:vcl-regret-1}.]
It suffices to prove Eq.~(\ref{eq:vcl-regret-1}) only, because the second inequality immediately follows by plugging in the definitions of $\alpha_{\zeta_t,t}^{i_t}$ and $\omega_{\zeta_t,t}^{i_t}$.

Combining Eqs.~(\ref{eq:mt-1}) and (\ref{eq:mt-2}) in Lemma \ref{lem:mt}, we have that 
\begin{align}
\mathbb E[R^T]
&= \sum_{t=1}^T \max_{i\in[n]} x_{it}^\top\theta - \mathbb E[x_{i_t,t}^\top\theta]
\leq \sum_{t=1}^T \mathbb E[\overline m_{0,t}-\overline m_{\zeta_t,t}] + \mathbb E[\overline m_{\zeta_t,t}-\underline m_{\zeta_t,t}].
\label{eq:rt-m}
\end{align}

Here the last inequality holds because $\max_{i\in[n]}x_{it}^\top\theta = \overline m_{0,t}$ and $x_{i_t,t}^\top\theta \geq \underline m_{\zeta_t,t}$ since $i_t\in\mathcal N_{\zeta_t,t}$. For each $t$, by Eq.~(\ref{eq:mt-1}), we have
\begin{align}
\mathbb E[\overline m_{0,t}-\overline m_{\zeta_t,t}]  \leq \sum_{\zeta = 0}^{\zeta_0} \mathbb E\left[\overline m_{\zeta,t}-\overline m_{\zeta+1,t}\right]  \lesssim \sqrt{d/T}. \label{eq:rt-m-a}
\end{align}

{
We next upper bound $\mathbb E[\overline m_{\zeta_t,t}-\underline m_{\zeta_t,t}]$. By Eq.~(\ref{eq:mt-2}), we have
\begin{align}
 \mathbb E[\overline m_{\zeta_t,t}-\underline m_{\zeta_t,t}]& \leq  \mathbb E\left[\max\{\overline m_{\zeta_t,t}-\underline m_{\zeta_t,t} - 2^{3-\zeta_t}, 0\} + 2^{3-\zeta_t}\right] \nonumber\\
 & \leq \sum_{\zeta = 0}^{\zeta_0}  \mathbb E\left[\max\{\overline m_{\zeta,t}-\underline m_{\zeta,t} - 2^{3-\zeta}, 0\}  \cdot \vct 1\{\zeta \leq \zeta_t\}\right] + \mathbb E \left[2^{3-\zeta_t}\right]\nonumber\\
 &\leq  \sum_{\zeta=0}^{\zeta_0} \frac{\sqrt{2\pi d}}{\zeta_0\sqrt{T}}+ \mathbb E \left[2^{3-\zeta_t}\right]\leq \frac{\sqrt{2\pi d}}{\sqrt{T}} +8 \left(\mathbb E\left[ \alpha_{\zeta_t,t}^{i_t}\omega_{\zeta_t,t}^{i_t}\right] + \sqrt{d/T}\right)\label{eq:rt-m-b-pre}\\
 & \lesssim \sqrt{d/T} + \mathbb E\left[ \alpha_{\zeta_t,t}^{i_t}\omega_{\zeta_t,t}^{i_t}\right].  \label{eq:rt-m-b}
\end{align}
In the above derivation, all steps are straightforward except for the second inequality in Eq.~(\ref{eq:rt-m-b-pre}), which we explain in more details here.
This inequality is derived by a case analysis on how $i_t$ is selected.
According to Algorithm \ref{alg:suplinucb}, there are only two cases in which $i_t$ is decided: the first clause or the third clause
in the ``if-elseif-else'' loop in Algorithm \ref{alg:suplinucb}. If the first clause is active, we have that $2^{3-\zeta_t} \leq 8 \sqrt{d/T}$. If the third clause is active, we have that $ \alpha_{\zeta_t,t}^{i_t}\omega_{\zeta_t,t}^{i_t}=\varpi_{\zeta,t}^{i_t}\geq 2^{-\zeta_t}$.
}

Combining Eqs.~(\ref{eq:rt-m},\ref{eq:rt-m-a},\ref{eq:rt-m-b}), we have
\begin{align*}
\mathbb E[R^T] \lesssim \mathbb E \sum_{t=1}^T \left(\sqrt{d/T} + \alpha_{\zeta_t,t}^{i_t}\omega_{\zeta_t,t}^{i_t} \right) \lesssim \sqrt{dT} + \mathbb E \sum_{t=1}^T  \alpha_{\zeta_t,t}^{i_t}\omega_{\zeta_t,t}^{i_t},
\end{align*}
which is to be demonstrated.
\end{proof}


\subsubsection{Removing the third $O(\sqrt{\log T})$ term}

In order to remove the third source of $O(\sqrt{\log T})$ term, our analysis goes one step beyond the classical elliptical potential analysis
(see Lemma \ref{lem:elliptical} in later sections)
to have more refined controls of the cumulative regret within each resolution level $\zeta$. 
More specifically, we establish the following main lemma upper bounding the sums of confidence band lengths:
\begin{lemma}
For any $\zeta$, let $\mathcal X_\zeta=\mathcal X_{\zeta,T}$ be all time periods $t$ such that $\zeta_t=\zeta$, and define $T_\zeta = |\mathcal X_{\zeta}|$.
Then the following hold with probability one:
\begin{align}
&\sum_{t\in\mathcal X_{\zeta}}\alpha_{\zeta,t}^{i_t}\cdot \omega_{\zeta,t}^{i_t}  \leq 2^{1-\zeta}T_\zeta\;\;\;\;\;\forall 0<\zeta\leq \zeta_0;\nonumber\\
&\sum_{t\in\mathcal X_{\zeta}}\alpha_{\zeta,t}^{i_t}\cdot \omega_{\zeta,t}^{i_t} 
\lesssim \sqrt{dT_{\zeta}\log(T_{\zeta})\log(e T/T_\zeta)\log n}\times\mathrm{poly}(\log\log(n T)) \;\;\;\;\;\forall 0\leq \zeta\leq\zeta_0;\nonumber\\
&T_\zeta \lesssim 4^\zeta d\log^4(nT) \;\;\;\;\;\forall 0<\zeta\leq \zeta_0.\nonumber
\end{align}
\label{lem:vcl-regret-2}
\end{lemma}

We remark on some interesting aspects of the results in Lemma \ref{lem:vcl-regret-2}.
First, we improve the ${\log(nT)}$ term that is common to previous elliptical potential lemma (Lemma \ref{lem:elliptical}) analysis to $\sqrt{\log(T_\zeta)\log(T/T_\zeta)\log n}$,
by exploiting the power of Lemma \ref{lem:vcl-regret-1} and an application of Jensen's inequality on $f(x)=\sqrt{x\ln(Tx/d)}$ instead of the more commonly used $f(x)=\sqrt{x}$.
We also impose an additional upper bound of $2^{1-\zeta}T_\zeta$ and an exponentially-increasing upper bound on $T_\zeta$ by carefully analyzing the procedures of Algorithm \ref{alg:suplinucb}.

\begin{proof}[Proof of Lemma \ref{lem:vcl-regret-2}.]
We prove the three inequalities in Lemma \ref{lem:vcl-regret-2} seperately.
We first prove, for all $\zeta>0$, that $\sum_{t\in\mathcal X_{\zeta}} \alpha_{\zeta,t}^{i_t}\omega_{\zeta,t}^{i_t}\leq 2^{1-\zeta}T_\zeta$.
Because $\zeta>0$, we have that $\alpha_{\zeta,t}^{i_t}\omega_{\zeta,t}^{i_t}=\varpi_{\zeta,t}^{i_t}\leq 2^{1-\zeta}$ for all $t\in\mathcal X_{\zeta}$ by the second clause of the if-elseif-else loop of Algorithm \ref{alg:suplinucb}. The inequality immediately follows.

We next prove the the second inequality in Lemma \ref{lem:vcl-regret-2}.
Below we state a version of the celebrated \emph{elliptical potential lemma}, key to many existing analysis of linearly parameterized bandit problems \citep{auer2002using,filippi2010parametric,abbasi2011improved,chu2011contextual,li2017provable}.
\begin{lemma}[\cite{abbasi2011improved}]
\label{lem:elliptical}
 For any vectors $y_1, y_2, \dots, y_T$, define $U_0 = I$ and $U_t = U_{t-1} + y_t y_t^\top$ for $t \geq 1$. It then holds that 
$
\sum_{t = 1}^{T} y_t^\top U_{t-1}^{-1}y_t 
\leq 2 \ln (\mathrm{det}(U_T)) .
$
\end{lemma} 

With Lemma \ref{lem:elliptical}, we can prove the second inequality by applying Jensen's inequality and the concavity of $f(x)=\sqrt{x}$ and $f(x)=\sqrt{x\log(Tx/d)}$.
More specifically, let $\mathcal X_{\zeta}^{+} = \{t \in \mathcal X_{\zeta} | \omega_{\zeta_t,t}^{i_t}\geq \sqrt{d/T}\}$, and $T_{\zeta}^+ = |\mathcal X_{\zeta}^{+}|$. 
Note that, for all $t\in\mathcal X_{\zeta}\backslash\mathcal X_{\zeta}^+$, because $\omega_{\zeta_t,t}^{i_t}<\sqrt{d/T}$,
it holds that $\max\{1, \log[T(\omega_{\zeta_t,t}^{i_t})^2/d]\} \leq 1$. Subsequently, 
by definition of $\alpha_{\zeta_t,t}^{i_t}$ and $\omega_{\zeta_t,t}^{i_t}$, we have
\begin{align}
 \sum_{t\in\mathcal X_{\zeta}}& \alpha_{\zeta_t,t}^{i_t} \omega_{\zeta_t,t}^{i_t}
\lesssim  \sqrt{\log (n\log T)}\cdot \sum_{t\in\mathcal X_\zeta} \sqrt{\max\{1,\log[T(\omega_{\zeta_t,t}^{i_t})^2/d]\}}\cdot \omega_{\zeta_t,t}^{i_t}\nonumber\\
&\leq \sqrt{\log(n\log T)} \left(T_\zeta \cdot \frac{1}{T_\zeta}\sum_{t\in\mathcal X_\zeta}\omega_{\zeta_t,t}^{i_t} +T_\zeta^+ \cdot \frac{1}{T_\zeta^+} \sum_{t\in\mathcal X_\zeta^+}\sqrt{\log[T(\omega_{\zeta_t,t}^{i_t})^2/d]}\cdot \omega_{\zeta_t,t}^{i_t} \right)\nonumber\\
&{\leq \sqrt{\log(n\log T)} \left(T_\zeta \cdot \frac{1}{T_\zeta}\sum_{t\in\mathcal X_\zeta}\omega_{\zeta_t,t}^{i_t} +T_\zeta^+  \sqrt{\frac{1}{T_\zeta^+} \sum_{t\in\mathcal X_\zeta^+}{\log[T(\omega_{\zeta_t,t}^{i_t})^2/d]}}\cdot \sqrt{\frac{1}{T_\zeta^+}\sum_{t\in\mathcal X_\zeta^+} (\omega_{\zeta_t,t}^{i_t})^2} \right)}\nonumber\\
&\leq  \sqrt{\log(n\log T)} \left(T_\zeta \sqrt{\frac{1}{T_\zeta}\sum_{t\in\mathcal X_\zeta}(\omega_{\zeta_t,t}^{i_t})^2} + T_\zeta^+ \sqrt{\log\left[\frac{T}{d}\frac{1}{T_\zeta^+}\sum_{t\in\mathcal X_\zeta^+}(\omega_{\zeta_t,t}^{i_t})^2\right]}\cdot \sqrt{\frac{1}{T_\zeta^+}\sum_{t\in\mathcal X_\zeta^+}(\omega_{\zeta_t,t}^{i_t})^2} \right).\nonumber
\end{align}
Here the second to last inequality holds by applying Cauchy-Schwarz inequality,
and the last inequality holds by applying Jensen's inequality and the concavity of $f(x)=\sqrt{x}$ and $f(x)=\sqrt{x\log(Tx/d)}$.
Applying Lemma \ref{lem:elliptical} to $\{\omega_{\zeta_t,t}^{i_t}\}_{t\in\mathcal X_\zeta}$ and noting that $\ln\det(\Lambda_{\zeta,T}) \leq d\ln(T_\zeta+1)$ because $\|x_{it}\|_2^2\leq 1$ for all $i,t$, we have 
\begin{align*}
 \sum_{t\in\mathcal T_{\zeta}} \alpha_{\zeta_t,t}^{i_t} \omega_{\zeta_t,t}^{i_t}
 \lesssim \sqrt{\log(n\log T)}\cdot \left(\sqrt{d T_\zeta \log (T_\zeta)} + \sqrt{\log[(T\log T_\zeta)/T_\zeta^+ ]}\cdot \sqrt{dT_\zeta^+ \log(T_\zeta)}\right) .
\end{align*}
Since $ T_\zeta^+ \log [(T \log T_\zeta) / T_\zeta^+] \leq T_\zeta + T_\zeta \log [(T \log T_\zeta) / T_\zeta]$ holds for $T_\zeta^+ \leq T_\zeta$, we further have
\begin{align*}
 \sum_{t\in\mathcal T_{\zeta}} \alpha_{\zeta_t,t}^{i_t} \omega_{\zeta_t,t}^{i_t}
 \lesssim \sqrt{\log(n\log T)}\cdot \left(\sqrt{d T_\zeta \log (T_\zeta)} + \sqrt{\log[(T\log T_\zeta)/T_\zeta ]}\cdot \sqrt{dT_\zeta \log(T_\zeta)}\right),
\end{align*}
which proves the second inequality in Lemma \ref{lem:vcl-regret-2}.
Note that, unlike the other two inequalities, this inequality holds for the first resolution level $\zeta=0$ as well.

We next prove the last inequality in Lemma \ref{lem:vcl-regret-2} which upper bounds $T_\zeta$ for $\zeta>0$.
By the second clause of the if-elseif-else line of Algorithm \ref{alg:suplinucb}, 
we know that $\varpi_{\zeta,t}^{i_t}=\alpha_{\zeta,t}^{i_t}\omega_{\zeta,t}^{i_t} \geq 2^{1-\zeta}$ for all $t\in\mathcal X_{\zeta}$.
Subsequently, 
\begin{equation*}
(2^{-\zeta-1})^2\cdot T_\zeta \leq \sum_{t\in\mathcal X_{\zeta}}(\varpi_{\zeta,t}^{i_t})^2 \leq \max_{t\in\mathcal X_{\zeta}}(\alpha_{\zeta,t}^{i_t})^2\cdot \sum_{t\in\mathcal X_{\zeta}}(\omega_{\zeta,t}^{i_t})^2
\lesssim \log(T/d)\cdot \log^2(n\log T)\cdot d\log T,
\end{equation*}
where the last inequality holds by applying Lemma \ref{lem:elliptical}.
Re-arranging the terms we obtain $T_\zeta \lesssim 4^{\zeta} d\log^4(nT)$, which is to be demonstrated. 
\end{proof}

\subsubsection{Putting it together}

We are now ready to combine Lemmas \ref{lem:vcl-regret-1}, \ref{lem:vcl-regret-2} to prove our main result in Theorem \ref{thm:upper-finite}.
We first divide the resolution levels $\zeta\in\{0,1,\cdots,\zeta_0\}$ into three different sets:
$\mathcal Z_0 := \{0\}$, 
$\mathcal Z_1 := \{1,\cdots,\zeta^*\}$ and $\mathcal Z_2 := \{\zeta:\zeta^*<\zeta\leq\zeta_0\}$, where $\zeta^*$ is an integer to be defined later.
Clearly $\mathcal Z_0$, $\mathcal Z_1$ and $\mathcal Z_2$ partition $\{0,\cdots,\zeta_0\}$.
The summation $\mathbb E[\sum_{t=1}^T\alpha_{\zeta_t,t}^{i_t}\cdot \omega_{\zeta_t,t}^{i_t}]$
on the right-hand side of Eq.~(\ref{eq:vcl-regret-1}) in Lemma \ref{lem:vcl-regret-1} can then be carried out separately (to simplify notations we denote $\gamma_{n,T}:= \mathrm{poly}(\log\log(nT))$; all inequalities below hold with probability one, following Lemma \ref{lem:vcl-regret-2}):
\begin{align}
\sum_{\zeta\in\mathcal Z_0}\sum_{t\in\mathcal X_\zeta}\alpha_{\zeta,t}^{i_t}\omega_{\zeta,t}^{i_t}
&\lesssim \gamma_{n,T}\sqrt{dT_0\log(T_0)\log(eT/T_0)\log n}\nonumber\\
&\lesssim \gamma_{n,T}\sqrt{dT\log T\log n};\label{eq:zeta-case0}\\
\sum_{\zeta\in\mathcal Z_1}\sum_{t\in\mathcal X_\zeta}\alpha_{\zeta,t}^{i_t}\omega_{\zeta,t}^{i_t} 
&\lesssim \sum_{\zeta=1}^{\zeta^*}2^{-\zeta}\cdot 4^{\zeta}d\log^4(nT)\cdot \gamma_{n,T}\leq 2^{\zeta^*+1}\cdot d\log^4(nT)\cdot  \gamma_{n,T};\label{eq:zeta-case1}\\
\sum_{\zeta\in\mathcal Z_2}\sum_{t\in\mathcal X_{\zeta}}\alpha_{\zeta,t}^{i_t}\omega_{\zeta,t}^{i_t} 
&\lesssim \sum_{\zeta \in \mathcal Z_2 }  \sqrt{dT_\zeta \log T \log(e T/T_\zeta) \log n}\cdot \gamma_{n,T} \ \nonumber\\
&\leq \sqrt{\left|\mathcal Z_2\right| d \left(\sum_{\zeta \in \mathcal Z_2} T_\zeta\right) \log  T \log \left(eT \cdot \frac{\left|\mathcal Z_2\right|}{ \sum_{\zeta \in \mathcal Z_2} T_\zeta}\right) \log n} \cdot \gamma_{n, T} \label{eq:zeta-case2-pre} \\
&\leq  \sqrt{\left|\mathcal Z_2\right| d T \log  T \log \left(e \left|\mathcal Z_2\right| \right) \log n} \cdot \gamma_{n, T}, \label{eq:zeta-case2}
\end{align}
Here some inequalities need more explanations.
Eq.~(\ref{eq:zeta-case0}) holds by a case analysis on $T_0$: if $T_0\leq T/\log T$ then 
$\sqrt{dT_0\log(T_0)\log(eT/T_0)\log n} \leq \sqrt{d\times (T/\log T)\times \log T\times \log(eT)\log n} \lesssim \sqrt{dT\log T\log n}$;
if, on the other hand, $T/\log T<T_0\leq T$, then 
$\sqrt{dT_0\log(T_0)\log(eT/T_0)\log n} \leq \sqrt{dT\log T\times \log(eT\log T/T)\log n} \lesssim \gamma_{n,T}\sqrt{dT\log T\log n}$.
Eq.~(\ref{eq:zeta-case2-pre}) holds by applying Jensen's inequality on the convex function $x\mapsto \sqrt{x \ln (eT / x)}$.
Eq.~\eqref{eq:zeta-case2} holds by applying the monotonicity of the function $x\mapsto\sqrt{x \ln (eT |\mathcal Z_2|/x)}$ and the fact that
$x:=\sum_{\zeta \in \mathcal Z_2} T_\zeta \leq T$. 

Recall that $\zeta_0=\lceil\sqrt{\log_2(T/d)}\rceil$ and therefore $\sqrt{T/d}\leq 2^{\zeta_0}\leq 2\sqrt{T/d}$. 
Select $\zeta^* = \zeta_0 - \lfloor 4\log_2e\cdot \ln\ln(nT)\rfloor$; we have that $|\mathcal Z_2| = O(\log \log (n T))$ and 
$2^{\zeta^*} \leq 2\sqrt{T}/(\sqrt{d}\ln^4(nT))$.
Adding Eqs.~(\ref{eq:zeta-case1},\ref{eq:zeta-case2}) and using Lemma \ref{lem:vcl-regret-1},
we obtain the main upper bound result of this paper (Theorem \ref{thm:upper-finite}).

\section{Lower bounds} \label{sec:lb}

In this section we establish our main lower bound result (Theorem \ref{thm:lower-finite}).
To simplify our analysis, we shall prove instead the following lower bound result, which places more restrictions
on the problem parameters $n,d$ and $T$:
\begin{theorem}\label{thm:LB-n-eq-expd}
Suppose $T\geq d^5$ and $n=2^{d/2}$. Then $\mathfrak R(T;n,d) = \Omega(1)\cdot d\sqrt{T\log T}$.
\end{theorem}

Theorem \ref{thm:LB-n-eq-expd} can be easily extended to the case of $n<2^{d/2}$ and $T<d^5$ as well,
by a zero-filling trick and reducing the effective dimensionality of the constructed adversarial instances.
We place the proof of this extension (which eventually leads to a proof of Theorem \ref{thm:lower-finite})
in Sec.~\ref{sec:additional-proof-sec-lb} and shall focus  on proving Theorem \ref{thm:LB-n-eq-expd} first.

In Sec.~\ref{sec:elliptical-lb}, we provide a short argument on the tightness of the elliptical potential lemma which is critically used in most existing analysis for linear bandit algorithms. This is done via a novel construction of the sequence $\{z_t\}$ and intuitively explains the necessity of an $O(\log n)$ factor in all known regret bounds whose analysis is based on the potential lemma. However, it requires several new ideas to show the desired lower bound for \emph{all} algorithms. 

{
In Sec.~\ref{sec:lb-d2}, we first prove  Theorem \ref{thm:LB-n-eq-expd} for the special case when $d=2$, as a warmup. We will demonstrate how to use the sequence $\{z_t\}$ to construct a collection of instances. Thanks to the properties of  $\{z_t\}$, a suboptimal pull at each round will contribute $\Omega(\sqrt{(\log T) / T}$ regret, and we will show that if we randomly choose a constructed adversarial instance, any policy will have $\Omega(1)$ probability to make a suboptimal pull at each round, and hence the $\Omega(\sqrt{T \log T})$ total regret is proved. 
}

In Sec.~\ref{sec:lb-general-d}, we extend the construction and analysis to general $d$. The construction for general $d$ is obtained via a direct-product fashion operation to $d/2$ copies of the adversarial instances constructed for $d=2$ so that there are $n=2^{d/2}$ arms at each round. We will show that if we randomly choose an adversarial instance from the constructed class, any policy will suffer $\Omega(d \sqrt{T \log T})$ regret.


\subsection{Tightness of the elliptical potential lemma} \label{sec:elliptical-lb}


To motivate our construction of adversarial bandit instances, in this section we give a warm-up exercise
showing the critical elliptical potential lemma (Lemma \ref{lem:elliptical})
used heavily in our analysis and existing analysis of linear contextual bandit problems
is in fact tight \citep{auer2002using,filippi2010parametric,abbasi2011improved,chu2011contextual,li2017provable},
even for the univariate case.



\begin{lemma}\label{lem_potential_low}
For any $T \geq 1$, there exists a sequence $z_1,z_2,\cdots z_T\in[0,1]$, such that if we let $V_0 = 1$ and $V_t = V_{t-1} + z_t z_t^{\top}$ for $t \geq 1$, then
\begin{align}\label{eq:potential-lb-1}
\sum_{t \in [T]} \sqrt{z_t^2/V_{t-1}} \geq \sqrt{\frac{T \cdot \ln T}{2}}. 
\end{align}
\end{lemma}

As a remark, using Lemma \ref{lem:elliptical} and Cauchy-Schwarz inequality we easily have 
$\sum_{t=1}^T \sqrt{z_t^2/V_{t-1}} \leq \sqrt{T}\cdot \sqrt{\sum_{t=1}^T z_t^2/V_{t-1}}
\leq \sqrt{T\ln V_T}\leq \sqrt{T\ln T}$ for all sequences $z_1,\cdots,z_T\in[0,1]$.
Lemma \ref{lem_potential_low} essentially shows that this argument cannot be improved,
and therefore current analytical frameworks of SupLinUCB \citep{chu2011contextual} or SupLinRel \citep{auer2002using}
cannot hope to get rid of all $O(\log T)$ terms.
While such an argument is not a rigorous lower bound proof as it only applies to specific analysis of certain policies,
we find still the results of Lemma \ref{lem_potential_low} very insightful,
which also inspires the construction of adversarial problem instances in our formal lower bound proof later.

\begin{proof}[Proof of Lemma \ref{lem_potential_low}.]
Let $S_t = \left(1 + \frac{\ln T}{2T}\right)^t$ for all $t \geq 0$, 
and let $z_t = \sqrt{\frac{S_{t-1} \ln T}{2T}}$ for all $t \geq 1$. 
Note that $z_t$ is a monotonically increasing function of $t$; and for $T \geq 1$ it holds that $S_{T-1} = \left(1 + \frac{\ln T}{2T}\right)^{T-1} \leq \sqrt{T}$. Therefore, for any $t \leq T$, we may verify that $z_t \in [0, 1]$ since 
\begin{align}\label{eq:zt-ub}
z_t \leq z_T = \sqrt{\frac{S_{T-1} \ln T}{2T}} 
\leq \sqrt{\frac{\sqrt{T} \ln T}{2 T}} < 1. 
\end{align}

Now we verify Eq.~\eqref{eq:potential-lb-1}. Note that for any $t \leq T$, we have
\begin{align}\label{eq:sum-of-seq}
    V_t = 1 + \sum_{j = 1}^{t} z_j z_j^{\top} = 1 + \frac{\ln T}{2T} \sum_{j=1}^{t} \left(1 + \frac{\ln T}{2T}\right)^{j-1} = \left(1 + \frac{\ln T}{2T}\right)^t = S_t .
\end{align} 
Therefore, we have 
\begin{align*}
~~~~~ \sum_{t=1}^T\sqrt{z_t^2/V_{t-1}} =  \sum_{t=1}^T |V_{t-1}^{-1/2} z_t| 
= \sum_{t=1}^T \sqrt{\frac{\ln T}{2T}} 
= \sqrt{\frac{T \ln T}{2}} .
\end{align*}

\end{proof}

{
\subsection{Warmup: the lower bound theorem for $d=2$}\label{sec:lb-d2}

In this subsection, we first prove the lower bound theorem (Theorem~\ref{thm:LB-n-eq-expd}) for $d = 2$ as a warmup. In this special case, there are only $n = 2^{d/2} = 2$ arms during each time period.

\subsubsection{Construction of adversarial problem instances} \label{sec:d2-adversarial-construction}



We will construct a finite set of bandit problem instances that will serve as the adversarial construction 
of our lower bound proof for general policies.
We start with the definitions of \emph{stages} and \emph{intervals}.

\begin{figure}[t]
\centering
\includegraphics[width=0.49\linewidth]{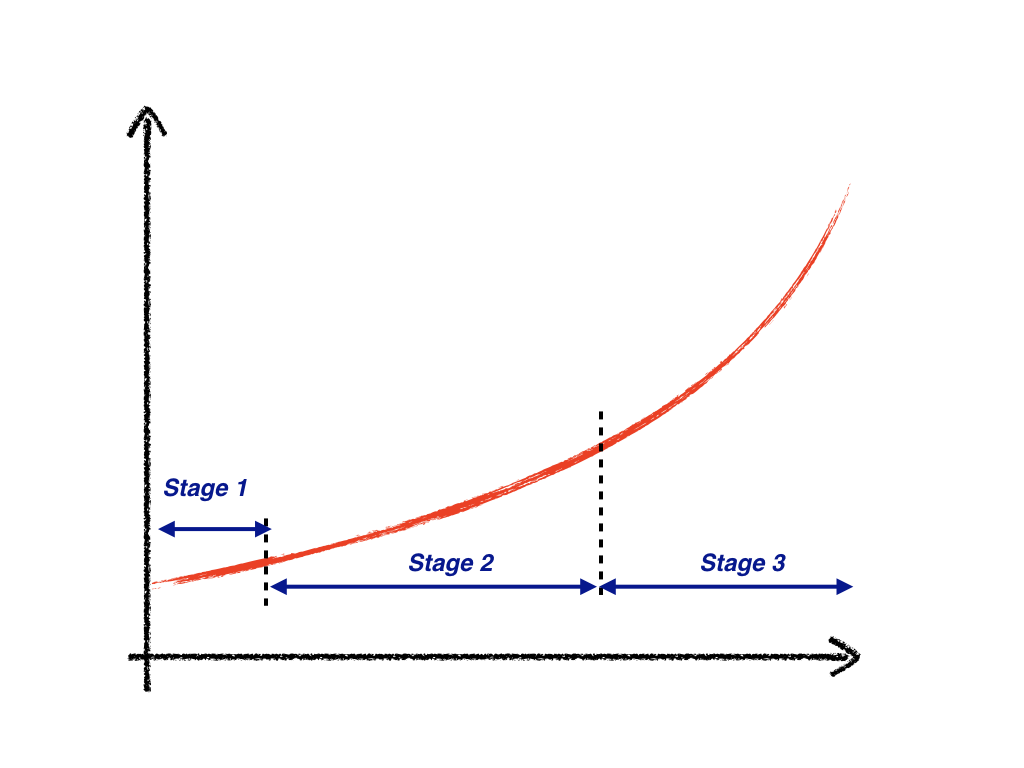}
\includegraphics[width=0.49\linewidth]{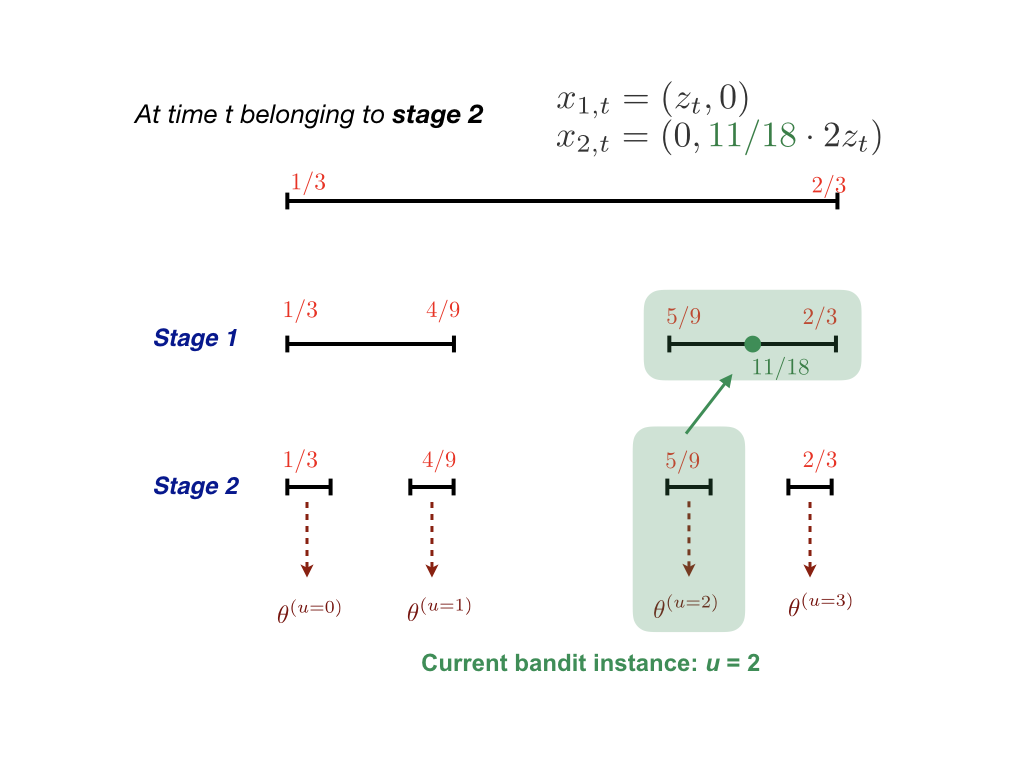}
\caption{{Illustration of the construction of adversarial bandit instances for the special case of $d=2$.
In the left panel, the entire $T$ time periods are partitioned into 3 \emph{stages}, with exponentially growing lengths.
The red curve is the $S_{t}$ sequence defined in the proof of Lemma \ref{lem_potential_low}.
In the right panel, \emph{intervals} for the first 2 stages are depicted, with the $j$th stage containing $2^j$ intervals.
The intervals are constructed in a similar way compared to the construction of the \emph{Cantor set}: within each new stage
the original interval is evenly partitioned into 3 parts and the middle part is left out.
Each of the leaf nodes/intervals then correspond to a regression model $\theta^{(u)}$ in a bandit instance.
The right figure also shows how the context vectors (in the case of $d=2$, there are $n=2$ context vectors per time period) are
 determined for a certain time period belonging to stage 2: the mid point of the interval in stage 1, which is the \emph{unique} parent
 of the leaf intervals corresponding to current regression models, is used to construct the context vectors.}}
\label{fig:lower-bound-d2}
\end{figure}

\paragraph{Stages.} We  use the $z_t$ sequence derived in Lemma~\ref{lem_potential_low} to divide the $T$ rounds of the bandit game into stages. Recall that $S_t = \left(1 + \frac{\ln T}{2T}\right)^t$ and $z_t = \sqrt{\frac{S_{t-1} \ln T}{2T}}$. We also have $S_t = 1 + \sum_{j = 1}^t z^2_j$ by Eq.~\eqref{eq:sum-of-seq}. 
Let $t_0 = 0$ and for each $j \geq 1$, let 
\begin{align}\label{eq:stj-up}
    t_j = \max\{t \leq T | S_t \leq 9^j \cdot S_{0}\}
\end{align}
mark the end of the $j$-th stage. In other words, stage $j$ ($j \geq 1$) consists of the rounds $t$ where  $t \in (t_{j-1}, t_j]$. Let $k = O(\ln T)$ be the last stage. We re-set $t_k = T$ so that the end of the last stage does not go beyond the time horizon. For notational consistency, we define stage $0$ to be the empty set of the rounds. We also remark that for $j \in \{1, 2, 3, \dots, k - 1\}$, we have that
\begin{align}\label{eq:stj-lb}
    S_{t_j} \geq 9^j \left(1 + \frac{\ln T}{2T}\right)^{-1} \geq 9^j \cdot \frac{1}{2}.
\end{align}

\paragraph{Intervals.} We also construct a set of intervals for each stage. In stage $0$, the set is $\cI_0 = \{I_{0,0} = [\frac13, \frac23]\}$.  
For each stage $j \geq 1$, the cardinality of $\cI_j$ is twice of the cardinality of $\cI_{j-1}$. 
More specifically, for each $I_{j-1,\xi} = [a, b] \in \cI_{j-1}$, we introduce the following two intervals to $\cI_j$,
\begin{align}\label{eq:def-interval}
I_{j, 2\xi} = \left[a, \frac{2 a + b}{3}\right] \subset I_{j-1, \xi} \text{~ and ~} I_{j, 2\xi+1} = \left[\frac{ a + 2b}{3}, b\right] \subset I_{j-1, \xi} .
\end{align}
We have $|\cI_{j}| = 2^j$ for all $j \in \{0, 1, 2, \dots, k\}$. For each $\xi \in \{0, 1, 2, \dots, 2^{k}-1\}$, we select an arbitrary real number $\gamma_{\xi}$ from the interval  $I_{k, \xi}$; we also let $\tau_{\xi}^{j}$ be the index of the unique interval in stage $j$ such that $\gamma_{\xi} \in I_{j, \tau_{\xi}^{j}}$. { (Note that the intervals in each stage are disjoint, so the uniqueness is guaranteed.)} Let $\alpha_{\xi}^{j}$ and $\beta_{\xi}^{j}$ be the two endpoints of $I_{j, \tau_{\xi}^{j}}$, i.e., $I_{j, \tau_{\xi}^{j}} = [\alpha_{\xi}^{j}, \beta_{\xi}^{j}]$. 

As shown in the right panel of Figure~\ref{fig:lower-bound-d2}, the intervals across the stages form a full binary tree where the root is $I_{0, 0}$. And the two children of an interval $I_{j-1, \xi}$ are $I_{j, 2\xi}$ and $I_{j, 2\xi+1}$. In this way, for each leaf node $I_{k, \xi}$ (where $\xi \in \{0, 1, 2, \dots, 2^k-1\}$), $[\alpha_{\xi}^{j}, \beta_{\xi}^{j}]$ is the ancestor interval of $I_{k, \xi}$ in the $j$-th stage, and $\tau_{\xi}^{j}$ is the index of the interval. We also have the following observation.

\begin{observation}\label{observ:xor}
Let $\oplus$ denote the binary bit-wise exclusive or (XOR) operator. For any leaf node $I_{k, \xi}$, and any stage $j$, if we let $\xi' = \xi \oplus 2^{k-j}$, we have that $\xi$ and $\xi'$ share the same ancestors at stages from $0$ to stage $(j-1)$, but have different ancestors at stage $j$.
\end{observation}

\paragraph{Bandit instances $B^{(u)}$.} Now we are ready to construct our lower bound instances. We will consider many bandit instances that are parameterized by $u \in  \{0, 1, 2, \dots, 2^{k}-1\}$. For each $u$, the bandit instance $B^{(u)}$ consists of a (2-dimensional) hidden vector $\theta^{(u)}$ and a set of context vectors $\{x_{i, t}^{(u)}\}_{i \in \{0, 1\}, t \in [T]}$. In all bandit instances, we set the noise $\epsilon_t$ to be independent Gaussian with variance $1$.

We first construct the hidden vectors $\theta^{(u)}$.  For each $u \in  \{0, 1, 2, \dots, 2^{k}-1\}$, we set $\theta^{(u)} = (\gamma_u, 1/2)$. By our construction, we have $\norm{\theta^{(u)}}_2 \leq 1$ for every $u$.  We then construct the set of context vectors $\{x_{i, t}^{(u)}\}_{i \in \{0, 1\}, t \in [T]}$. For any round $t$ that belongs to stage $j$, we set $x_{0, t}^{(u)} = (z_t, 0)$ and  $x_{1, t}^{(u)} =(0, ((\alpha_{u}^{j-1}+ \beta_{u}^{j-1})/2) \cdot 2 z_t)$. On may verify (using Eq.~\eqref{eq:zt-ub}) that for sufficiently large $T$, we have that  $\norm{x^{(u)}_{i,t}}_2 \leq 1$ for all $i$ and $t$. 

\subsubsection{The analysis for suboptimal pulls}

Since there are only two candidate actions at each time $t$, we call the action with the smaller expected reward to be a \emph{suboptimal pull}. The following lemma lower bounds the regret incurred by each suboptimal pull.

\begin{lemma}\label{claim:d2-suboptimal-pull-regret}
For any instance $B^{(u)}$ and any time $t$, if a policy makes a suboptimal pull at time $t$, then the incurred regret is at least ${\sqrt{\ln T}}/({36 \sqrt{T}})$.
\end{lemma}
\begin{proof}
Assuming time $t$ is in stage $j$, we have that the incurred regret is 
\[
\left| (z_t, 0)^\top \theta^{(u)} - (0, ((\alpha_{u}^{j-1}+ \beta_{u}^{j-1})/2) \cdot 2 z_t)^\top\theta^{(u)}\right| = z_t \cdot \left| \gamma_u - \frac{\alpha_{u}^{j-1}+ \beta_{u}^{j-1}}{2} \right| \geq z_t \cdot \frac{|\alpha_{u}^{j-1}- \beta_{u}^{j-1}|}{6}.
\]
Since $z_t = \sqrt{\frac{S_{t-1} \ln T}{2T}}$ and $\left|\alpha_{u}^{j-1}- \beta_{u}^{j-1}\right| = 3^{-j}$, we have 
\begin{align*}
z_t \cdot \frac{|\alpha_{u}^{j-1}- \beta_{u}^{j-1}|}{6} =
\frac{1}{6 \cdot 3^{j}} \sqrt{\frac{S_{t-1} \ln T}{2T}}
\geq \frac{1}{6 \cdot 3^{j}}\sqrt{\frac{S_{t_{j-1}}\ln T}{2T}}
\geq   \frac{\sqrt{\ln T}}{36 \sqrt{T}},
\end{align*}
where the last inequality is because of Eq.~\eqref{eq:stj-lb}. 
\end{proof}

For any policy $\pi$, underlying model $\theta^{(u)}$, let $p_t^{u, \pi}$ be the probability that $\pi$ makes a suboptimal pull at time $t$. We have the following corollary.
\begin{corollary}\label{cor:d2-suboptimal-pull-regret}
$\displaystyle{\mathbb E[R^T] \geq \sum_{t=1}^{T} p_t^{u, \pi} \cdot \sqrt{\ln T}/(36 \sqrt{T})}$.
\end{corollary}

The following lemma lower bounds $p_t^{u, \pi}$.
\begin{lemma}\label{lem:d2-subopt-pull-large-in-neighbors}
For any stage $j$, let  $u, u' \in \{0, 1, 2, \dots, 2^{k}-1\}$
be two parameters such that $\tau_{u}^{j-1} = \tau_{u'}^{j-1}$ but $\tau_{u}^{j} \neq \tau_{u'}^{j}$. (The definition of $\tau$ can be found in Sec.~\ref{sec:d2-adversarial-construction}.)
Then for any  policy $\pi$ and time period $t$ in stage $j$, it holds that $p_{t}^{u,\pi}+p_{t}^{u',\pi} \geq 1/2$.
\end{lemma}

Intuitively, Lemma~\ref{lem:d2-subopt-pull-large-in-neighbors} holds because at any time in and before stage $j$, all context vectors in $B^{(u)}$ are the same as $B^{(u')}$ and no policy $\pi$ shall be able to distinguish $\theta^{(u)}$ from $\theta^{(u')}$
since $u$ and $u'$ are close, and therefore at least one of the probabilities $p_{t}^{u,\pi}$ or $p_{t}^{u',\pi}$ shall be large. The formal proof of Lemma~\ref{lem:d2-subopt-pull-large-in-neighbors} involves the application of \emph{Pinsker's inequality} \citep{pinsker1960information}, and is presented as follows.

\begin{proof}[Proof of Lemma~\ref{lem:d2-subopt-pull-large-in-neighbors}]
By our construction, we have that 
\[
\left|(1, 0)^\top (\theta^{(u)} - \theta^{(u')})\right| \leq \left(\frac{1}{3}\right)^{j} .
\] 
Therefore, by Claim~\ref{claim:d2-event-prob-difference} (proved below), for any event $E$ at time $t$ before the end of stage $j$,  we have that
\begin{equation}\label{eq:d2-prob-up}
\left| \Pr\left[E | u \right] - \Pr\left[E | u' \right] \right| 
\leq \frac{1}{2} \left|(1, 0)^\top (\theta^{(u)} - \theta^{(u')})\right|  \sqrt{S_t}
\leq \frac{1}{2} \cdot \left(\frac{1}{3}\right)^{j} \sqrt{S_t} \leq \frac{1}{2}. 
\end{equation}
The last inequality holds because at any time $t$ in stage $j$, it holds that $S_t \leq 9^j$. 

Note that, at any time $t$, if the model parameter is $u$, the difference between the rewards of the two possible actions is
\[
(x_{0, t}^{(u)} - x_{1, t}^{(u)})^{\top} \left(\theta_{2s-1}^{(U)}, \theta_{2s}^{(U)}\right) = \frac{z_t}{2} \left(2 \gamma_{u} - \alpha_{u}^{j-1}- \beta_{u}^{j-1}\right).
\]
This value is greater than $0$ if and only if $2 \gamma_{u} > \alpha_{u}^{j-1}+ \beta_{u}^{j-1}$. Since $\tau_{u}^{j-1} = \tau_{u'}^{j-1}$ and $\tau_{u}^j \neq  \tau_{u'}^j$, by our construction Eq.~\eqref{eq:def-interval}, we have that exactly one of $\gamma_{u}$ and $\gamma_{u'}$ is greater than $\frac{1}{2}({\alpha_{u}^{j-1}+ \beta_{u}^{j-1}})$. In other words, at time $t$, any arm that is  suboptimal  for  parameter is $u$ is not  suboptimal  for  parameter $u'$, and vice versa. In light of this, let $E$ be the event that at time $t$ policy $\pi$ pulls an arm that is $s$-suboptimal for parameter $u$, and we have that the complement event $\bar{E}$ is that at time $t$ policy $\pi$ pulls an arm that is  $s$-suboptimal for parameter $u'$. By Eq.~\eqref{eq:d2-prob-up}, we have
\begin{align*}
p_{t}^{u,\pi}+p_{t}^{u',\pi} = \Pr\left[E |u \right] + \Pr\left[\bar{E} | u' \right] = 1 + \Pr\left[E |u \right] - \Pr\left[E | u' \right]  \geq \frac{1}{2} .
\end{align*}

\end{proof}

\begin{claim}\label{claim:d2-event-prob-difference}
For any $u, u'$, let $j$ be the largest number such that $\tau_{u}^{j-1} = \tau_{u'}^{j-1}$. For any time $t \leq t_j$ and any event $E$ that happens at time $t$, we have 
\[
\left| \Pr\left[E | u\right] - \Pr\left[E | u'\right] \right| 
\leq \frac{1}{2} \left|(1, 0)^\top (\theta^{(u)} - \theta^{(u')}) \right| \sqrt{S_t} .
\]
\end{claim}

\begin{proof}
Note that by our construction, at any time $t\leq t_j$, the contextual vectors of both $B^{(u)}$ and $B^{(u')}$ are the same.
Moreover, for any hidden vector and any arm, the reward distribution is a shifted standard Gaussian with variance $1$. 

For any time $t \leq t_j$, let $D_1$ be the product of the arm reward distributions at and before round $t$ when the hidden vector is $\theta^{(u)}$, and let $D_2$ be the same product distribution when the hidden vector is $\theta^{(u')}$. 
Since the second dimensions of $\theta^{(u)}$ and $\theta^{(u')}$ are the same, the differences of the mean rewards at any time $t' : 1 \leq t' \leq t$ for $\theta^{(u)}$ and  $\theta^{(u')}$ is either $\theta^{(u)}_{1} z_{t'} - \theta^{(u')}_{1} z_{t'}$ (if the first arm is pulled) or $0$ (if the second arm is pulled), where $\theta^{(u)}_{1}$ and $\theta^{(u')}_{1}$ denote the first dimensions of the corresponding vectors. Note that the KL divergence between two variance-$1$ Gaussians with means $\mu_1$ and $\mu_2$ is $|\mu_1 - \mu_2|^2/2$. Therefore, we have 
\begin{multline*}
\mathrm{KL} \left(D_1 \| D_2\right) \leq   \frac{1}{2} \sum_{t'=1}^{t} \left|\theta^{(u)}_{1} z_{t'} - \theta^{(u')}_{1}z_{t'} \right|^2 
= \frac{1}{2} \left| (1, 0)^\top (\theta^{(u)} - \theta^{(u')})\right|^2 \sum_{t'=1}^{t} z^2_{t'} \\
\leq \frac{1}{2} \left|(1, 0)^\top (\theta^{(u)} - \theta^{(u')}) \right|^2 \left(1 + \sum_{t'=1}^{t} z^2_{t'} \right)
 = \frac{1}{2} \left|(1, 0)^\top (\theta^{(u)} - \theta^{(u')}) \right|^2 S_t. 
\end{multline*}
Therefore, at time $t$, and for any event $E$,  we have
\[
\left| \Pr\left[E | u \right] - \Pr\left[E | u'\right] \right| 
\leq \sqrt{\frac{1}{2} \mathrm{KL}(D_1 \| D_2)} 
\leq \frac{1}{2} \left| (1, 0)^\top (\theta^{(u)} - \theta^{(u')}) \right| \sqrt{S_t}
\]
where the first inequality holds because of Pinsker's inequality (Lemma~\ref{lem:pinsker}). 
\end{proof}

\subsubsection{The average-case analysis}

We are now ready to prove Theorem \ref{thm:LB-n-eq-expd} assuming $d=2$. 
Recall that $\{0, 1, 2, \dots, 2^k-1\}$ is the finite collection of the parameters $u$ for the adversarial bandit instances we constructed in Sec.~\ref{sec:d2-adversarial-construction},
and $p_{t}^{U,\pi}$ is the probability of an suboptimal pull at time $t$.
The minimax regret $\mathfrak R(T;n=2,d=2)$ can then be lower bounded by 

\begin{align}
\mathfrak R(T;2,2) \geq \inf_{\pi} \max_{U\in\mathcal U} \mathbb E[R^T] 
&\geq  \inf_{\pi} \max_{u\in\{0, 1, 2, \dots, 2^k-1\}}\frac{\sqrt{\ln T}}{36\sqrt{T}}\cdot \sum_{t=1}^T p_{t}^{U,\pi}\label{eq:d2-average-reduction-1}\\
&\geq \inf_{\pi}\frac{1}{2^k}\sum_{u=0}^{2^k-1} \frac{\sqrt{\ln T}}{36\sqrt{T}}\cdot \sum_{t=1}^T p_{t}^{U,\pi}.\label{eq:d2-average-reduction-2}
\end{align}

Here, Eq.~(\ref{eq:d2-average-reduction-1}) holds by applying Corollary \ref{cor:d2-suboptimal-pull-regret},
and Eq.~(\ref{eq:d2-average-reduction-2}) holds because the average regret always lower bounds the worst-case regret.

Recall that $\oplus$ denotes the binary bit-wise exclusive or (XOR) operator. For any time $t$, suppose it is in stage $j$, for any parameter $u$, if we let $u' = u \oplus 2^{k-j+1}$, by Lemma~\ref{lem:d2-subopt-pull-large-in-neighbors} and Observation~\ref{observ:xor}, we have that $p_{t}^{U,\pi} + p_{t}^{U',\pi}\geq 1/2$ for all policies $\pi$.
Let $q_{j}^{u,\pi}$ be the expected number of suboptimal pulls made by policy $\pi$ in all time periods of stage $j$. 
Then 
\begin{equation}\label{eq:d2-q-lb}
q_{j}^{U,\pi} + q_{j}^{U',\pi} \geq \frac{1}{2}(t_j - t_{j-1}), \;\;\;\;\;\;\forall \text{ policy $\pi$.}
\end{equation}

We next compute the average expected number of suboptimal pulls made by any policy $\pi$ over all time periods $t$.
{
\begin{align}
\frac{1}{2^k}\sum_{u=0}^{2^k-1} \sum_{t=1}^T\sum_{s=1}^{d/2} p_{t}^{u,\pi}
=&\frac{1}{2^k}\sum_{u=0}^{2^k-1}\sum_{j=1}^{k} q_{j}^{u,\pi} \nonumber\\
=& \frac{1}{2 \cdot 2^k}\sum_{j=1}^{k}  \sum_{u = 0}^{2^k-1}  \left( q_{j}^{u,\pi} + q_{j}^{u \oplus 2^{k-j+1},\pi} \right)\geq \frac{1}{4 \cdot 2^k} \sum_{j=1}^{k}  \sum_{u = 0}^{2^k-1}  (t_j - t_{j-1})
= \frac{T}{4},\label{eq:d2-final-average}
\end{align}
}
where the inequality holds because of Eq.~\eqref{eq:d2-q-lb}. 
Combining Eqs.~(\ref{eq:d2-final-average},\ref{eq:d2-average-reduction-2}) we complete the proof of Theorem \ref{thm:LB-n-eq-expd} for $n=2$ and $d=2$.

}

\subsection{The lower bound proof for general $d$} \label{sec:lb-general-d}

\subsubsection{Construction of adversarial problem instances} \label{sec:gend-adversarial-construction}

We will use the definitions of \emph{stages} and \emph{intervals} introduced in Sec.~\ref{sec:d2-adversarial-construction}, and will need to define \emph{dimension groups} as follows.

\paragraph{Dimension groups.} Without loss of generality, we assume that $d$ is an even number. 
We also divide the $d$ dimensions into $d/2$ groups, where the $s$-th group ($s \in \{1, 2, 3, \dots, d/2\}$) corresponds to the $(2s - 1)$-th and $2s$-th dimension. 

Now we describe the set of adversarial instances for general $d$. The set of the instances can be viewed as a $(d/2)$-time direct product of the $d = 2$ special case.

\paragraph{Bandit instances $\mathcal{B}^{(U)}$.} Now we are ready to construct our lower bound instances. We will consider many bandit instances that are parameterized by $U = (u_1, u_2, \dots, u_{d/2})$ where each $u_s$ ($s \in [d/2]$) is indexed by $\{0, 1, 2, \dots, 2^{k}-1\}$. Let $\mathcal U$ denote the set of all possible $U$'s, and we have $|\mathcal U| = 2^{kd/2}$. For each $U \in \mathcal U$, the bandit instance $\mathcal{B}^{(U)}$ consists of a hidden vector $\theta^{(U)}$ and a set of context vectors $\{x_{i, t}^{(U)}\}$. In all bandit instances, we set the noise $\epsilon_t$ to be independent Gaussian with variance $1$.

We first construct the hidden vectors $\theta^{(U)}$.  For each $U = (u_1, u_2, \dots, u_{d/2})$, we set $\theta^{(U)}_{2 s - 1} = \frac{\gamma_{u_s}}{\sqrt{d}}$ and $\theta^{(U)}_{2s} = \frac{1}{2\sqrt{d}}$ for all $s \in [d/2]$. By our construction, we have $\norm{\theta^{(U)}}_2 \leq 1$ for every $U$.

We then construct the set of context vectors $\{x_{i, t}^{(U)}\}$. We label the $n$ arms by $0, 1, 2, \dots, 2^{d/2} - 1$. For each arm $i \in \{0, 1, 2, \dots, 2^{d/2} - 1\}$ and each dimension group $s \in [d/2]$, let $b_s(i)$ be the $s$-th least significant bit in the binary representation of $i$. At any time $t$, the context vectors of two arms $i_1$ and $i_2$ may differ at the $s$-th dimension group only when $b_s(i_1) \neq b_s(i_2)$. For any round $t$ that belongs to stage $j$, and for each dimension group $s \in [d/2]$, we set the corresponding entries of the context vector of Arm $i$ to be $((x^{(U)}_{i,t})_{2s-1}, (x^{(U)}_{i, t})_{2s}) = (z_t \cdot \sqrt{d}, 0)$ if $b_s(i) = 0$, and set the corresponding entries to be $((x^{(U)}_{i,t})_{2s-1}, (x^{(U)}_{i, t})_{2s}) = (0, ((\alpha_{u_s}^{j-1}+ \beta_{u_s}^{j-1})/2) \cdot 2 z_t \cdot \sqrt{d})$ if $b_s(i) = 1$. 
For $T \geq d^5$, one may easily verify (using Eq.~\eqref{eq:zt-ub}) that $\norm{x^{(U)}_{i,t}}_2 \leq 1$ for all $i$ and $t$.

\subsubsection{$s$-suboptimal pulls and their implications}\label{sec:gend-s-suboptimal-pull}

In our construction of adversarial bandit instances $\mathcal B^{(U)}$, 
for each dimension group $s\in[d/2]$ the policy has to choose between 
two potential actions (corresponding to this group $s$) of $((x^{(U)}_{i,t})_{2s-1}, (x^{(U)}_{i, t})_{2s}) = (0, (\alpha_{u_s}^{j-1}+ \beta_{u_s}^{j-1}) z_t \cdot \sqrt{d})$ and $((x^{(U)}_{i,t})_{2s-1}, (x^{(U)}_{i, t})_{2s}) = (z_t \cdot \sqrt{d}, 0)$.
One of the actions would lead to a larger expected reward depending on the unknown model $\theta$, 
and a policy should try to identify and execute such action for as many times as possible.
This motivates us to define the concept of \emph{$s$-suboptimal pulls}, which counts the number of times a policy plays
a suboptimal action.
\begin{definition}[$s$-suboptimal pull]
For any $s\in[d/2]$,
we say a policy makes one \emph{$s$-suboptimal pull} at time period $t$ if the policy picks an action corresponding to 
the lesser expected reward.
\end{definition}

We also break up the regret incurred by a policy at time period $t$ into $d/2$ terms, each corresponding 
to a dimension group $s\in[d/2]$.
\begin{definition}[$s$-segment regret]
For any $s\in[d/2]$ and time period $t$, define 
$$
r_s^{(t)} = \left(\left((x_{i_t^*, t}^{(U)})_{2s-1}, (x_{i_t^*, t}^{(U)})_{2s}\right) - \left((x_{i_t, t}^{(U)})_{2s-1}, (x_{i_t, t}^{(U)})_{2s}\right)\right)^\top \left(\theta^{(U)}_{2s-1}, \theta^{(U)}_{2s}\right),
$$
where $i_t$ is the action the policy plays and $i_t^*$ is the optimal action at time $t$.
\end{definition}

By definition, the regret incurred at time period $t$ can be expressed as $\sum_{s=1}^{d/2}r_s^{(t)}$.
Also, intuitions behind the definition of $s$-optimal pulls suggest that the more $s$-optimal pulls a policy makes,
the larger $s$-segment regret it should incur.
The following lemma quantifies this intuition by giving a lower bound of $s$-segment regret using the number of $s$-optimal pulls. The lemma is an analogue of Lemma~\ref{claim:d2-suboptimal-pull-regret}  and its proof is deferred to Sec.~\ref{sec:additional-proof-sec-lb}.
\begin{lemma}\label{claim:s-suboptimal-pull-regret} 
For any instance $\mathcal{B}^{(U)}$, any coordinate group $s$, and any time $t$, if a policy makes an $s$-suboptimal pull at time $t$, then $r_s^{(t)} \geq {\sqrt{\ln T}}/({36 \sqrt{T}})$.
\end{lemma}

As a corollary, the expected regret of a policy $\pi$ can be explicitly lower bounded
by the expected number of $s$-optimal pulls the policy makes.
\begin{corollary}\label{cor:s-suboptimal-pull-regret}
For policy $\pi$, underlying model $\theta^{(U)}$ and dimension group $s$,
let $p_{s,t}^{U,\pi}$ be the probability of an $s$-suboptimal pull at time $t$.
Then $\mathbb E[R^T] \geq \sum_{t=1}^T \sum_{s=1}^{d/2} p_{s,t}^{U,\pi}\cdot \sqrt{\ln T}/(36\sqrt{T})$.
\label{cor:s-suboptimal-pull-regret}
\end{corollary}

\subsubsection{Average-case analysis}\label{sec:gend-average-case-pinsker}

Recall that $\mathcal U$ is the finite collection of the parameters $U$ for the adversarial bandit instances we constructed in Sec.~\ref{sec:gend-adversarial-construction},
and $p_{s,t}^{U,\pi}$ be the probability of an $s$-suboptimal pull at time $t$ defined in Sec.~\ref{sec:gend-s-suboptimal-pull}.
The minimax regret $\mathfrak R(T;n,d)$ can then be lower bounded by 

\begin{align}
\mathfrak R(T;n,d) \geq \inf_{\pi\in\Pi_{T,n,d}} \max_{U\in\mathcal U} \mathbb E[R^T] 
&\geq  \inf_{\pi\in\Pi_{T,n,d}} \max_{U\in\mathcal U}\frac{\sqrt{\ln T}}{36\sqrt{T}}\cdot \sum_{t=1}^T\sum_{s=1}^{d/2} p_{s,t}^{U,\pi}\label{eq:average-reduction-1}\\
&\geq \inf_{\pi\in\Pi_{T,n,d}}\frac{1}{|\mathcal U|}\sum_{U\in\mathcal U} \frac{\sqrt{\ln T}}{36\sqrt{T}}\cdot \sum_{t=1}^T\sum_{s=1}^{d/2} p_{s,t}^{U,\pi}.\label{eq:average-reduction-2}
\end{align}

Here, Eq.~(\ref{eq:average-reduction-1}) holds by applying Corollary \ref{cor:s-suboptimal-pull-regret},
and Eq.~(\ref{eq:average-reduction-2}) holds because the average regret always lower bounds the worst-case regret.

The following lemma lower bounds $\{p_{t,s}^{U,\pi}\}$ for particular pairs of parameterizations.
\begin{lemma}\label{lem:subopt-pull-large-in-neighbors}
For any stage $j$ and any group $s$, let $U = (u_1, u_2, \dots, u_{d/2})$ and $U' = (u_1', u_2', \dots, u_{d/2}')$ 
be two parameters such that $\tau_{u_s}^{j-1} = \tau_{u_s'}^{j-1}$ but $\tau_{u_s}^{j} \neq \tau_{u_s'}^{j}$, and $u_{a} = u_{a}'$ for every $a \neq s$.
(The definition of $\tau$ can be found in Sec.~\ref{sec:d2-adversarial-construction}.)
Then for any policy $\pi$ and time period $t$ in stage $j$, it holds that $p_{s,t}^{U,\pi}+p_{s,t}^{U',\pi} \geq 1/2$.
\end{lemma}
Lemma~\ref{lem:subopt-pull-large-in-neighbors} is an analogue of Lemma~\ref{lem:d2-subopt-pull-large-in-neighbors}  and its proof is deferred to Sec.~\ref{sec:additional-proof-sec-lb}.

We are now ready to prove Theorem \ref{thm:LB-n-eq-expd}.
For any parameter $U = (u_1, u_2, \dots, u_{d/2})$ and any dimension group $s$, we may write $U  = (u_s, u_{-s})$ where $u_{-s} = (u_1, \dots, u_{s-1}, u_{s+1}, \dots, u_{d/2})$. Recall that $\oplus$ denotes the binary bit-wise exclusive or (XOR) operator. For any time $t$, suppose it is in stage $j$. If we let $U' = (u_s \oplus 2^{k-j+1}, u_{-s})$, by Lemma~\ref{lem:subopt-pull-large-in-neighbors} and Observation~\ref{observ:xor}, we have that $p_{s,t}^{U,\pi} + p_{s,t}^{U',\pi}\geq 1/2$ for all  policies $\pi$.
Let $q_{s,j}^{U,\pi}$ be the expected number of $s$-suboptimal pulls made by policy $\pi$ in all time periods of stage $j$. 
Then 
\begin{equation}\label{eq:q-lb}
q_{s,j}^{U,\pi} + q_{s,j}^{U',\pi} \geq \frac{1}{2}(t_j - t_{j-1}), \;\;\;\;\;\;\forall \text{ policy $\pi$.}
\end{equation}

We next compute the average expected number of $s$-suboptimal pulls made by any policy $\pi$ over all time periods $t$.
{
\begin{align}
\frac{1}{|\mathcal U|}\sum_{U\in\mathcal U} \sum_{t=1}^T\sum_{s=1}^{d/2} p_{s,t}^{U,\pi}
&=\frac{1}{|\mathcal U|}\sum_{U\in\mathcal U} \sum_{j=1}^{k} \sum_{s = 1}^{d/2} q_{s, j}^{U,\pi} 
= \frac{1}{|\mathcal U|}\sum_{j=1}^{k} \sum_{s = 1}^{d/2}\sum_{u_{-s}}  \sum_{u_s = 0}^{2^k-1}  q_{s, j}^{(u_s, u_{-s}),\pi} \nonumber\\
=& \frac{1}{2|\mathcal U|} \sum_{j=1}^{k} \sum_{s = 1}^{d/2}\sum_{u_{-s}}  \sum_{u_s = 0}^{2^k-1}  \left( q_{s, j}^{(u_s, u_{-s}),\pi} + q_{s, j}^{(u_s \oplus 2^{k-j+1}, u_{-s}),\pi} \right)\nonumber\\
\geq& \frac{1}{4|\mathcal U|} \sum_{j=1}^{k} \sum_{s = 1}^{d/2}\sum_{u_{-s}}  \sum_{u_s = 0}^{2^k-1}  (t_j - t_{j-1})
= \frac{T}{4} \cdot \frac{d}{2},\label{eq:final-average}
\end{align}
}
where the inequality holds because of Eq.~\eqref{eq:q-lb}. 
Combining Eqs.~(\ref{eq:final-average},\ref{eq:average-reduction-2}) we complete the proof of Theorem \ref{thm:LB-n-eq-expd}.

%
%
%
%
%

\subsection{Additional proofs}\label{sec:additional-proof-sec-lb}

We first remark that the $T\geq d^5$ condition in Theorem \ref{thm:LB-n-eq-expd} can be relaxed to $T \geq d^{2 + \eps}$ for any small constant $\eps > 0$, which leads to the following theorem.

\begin{theorem} \label{thm:lb-n-eq-expd-strengthened}
For any small constant $\epsilon > 0$ and sufficiently large $d$, suppose $T\geq d^{2+\eps}$ and $n=2^{d/2}$. Then $\mathfrak R(T;n,d) = \Omega(1)\cdot d\sqrt{\epsilon T\log T}$.
\end{theorem}

Theorem~\ref{thm:lb-n-eq-expd-strengthened} can be proved using the identical argument of the proof of Theorem~\ref{thm:LB-n-eq-expd}, where the difference is that first we redefine 
\[
S_t = \left(1 + \frac{(\epsilon/2) \cdot \ln T}{2T}\right)^t \text{~and~} z_t = \sqrt{\frac{(\epsilon/2) \cdot S_{t-1} \ln T}{2T}}
\]
for all $t \in [T]$. Then we list the following changes to the calculations in the proof of Theorem~\ref{thm:LB-n-eq-expd}.

\begin{itemize}

\item In \eqref{eq:zt-ub}, we have 
\[
z_t \leq z_T \leq \sqrt{\frac{(\epsilon/2) S_{T-1} \ln T}{2T}} \leq \sqrt{\frac{\epsilon}{4} \cdot T^{\epsilon/4 - 1} \ln T} .
\]

\item At the end of Sec.~\ref{sec:gend-adversarial-construction},, we verify that $\| x_{i, t}^{(U)} \|_2 \leq 1$ since
\[
\| x_{i, t}^{(U)} \|_2^2 = \frac{d}{2} \cdot 4d \cdot \frac{\epsilon}{4} \cdot T^{\epsilon/4 - 1} \ln T,
\]
 and for $T \geq d^{2+\epsilon}$, this value is at most $d^{\epsilon^2/4 - \epsilon/2} \ln d \cdot \frac{(2+\epsilon) \epsilon}{2} \leq 1$ (for large enough $d$).

\item At the end of the proof of Lemma~\ref{claim:s-suboptimal-pull-regret}, we have
\[
\eqref{eq:s-suboptimal-pull-regret-pre} = \frac{1}{6 \cdot 3^{j+1}} \sqrt{\frac{(\epsilon/2) S_{t-1} \ln T}{2T}} \geq \frac{\sqrt{\epsilon \ln T}}{36\sqrt{2T}}.\]
 Therefore, the corresponding lower bounds in Lemma~\ref{claim:s-suboptimal-pull-regret}, Corollary~\ref{cor:s-suboptimal-pull-regret}, and the final regret lower bound in the theorem will be multiplied by a factor of $\sqrt{\epsilon/2}$.
\end{itemize}

We also remark that the requirement that $T \geq d^{2 + \epsilon}$ is essentially necessary for the  $\Omega ( d \sqrt{T \log T} )$ regret lower bound. Indeed, if $T \leq d^2$, we have $\Omega(d \sqrt{T \log T}) \geq \Omega(T \sqrt{\log T}) = \omega(T)$, while the regret of any algorithm is at most $T$.

We now use in Theorem~\ref{thm:lb-n-eq-expd-strengthened} to establish the regret lower bound for $n \leq 2^d$, proving Theorem~\ref{thm:lower-finite}.

\medskip
\begin{proof}[Proof of Theorem~\ref{thm:lower-finite}]
{To simplify the presentation, we assume without loss of generality that $n$ is an integer power of $2$ and $d$ is a multiple of $\log_2 n$.
We divide the time horizon into $\frac{d}{\log_2 n}$ phases, where
phase $j \in [d/\log_2 n]$ is consists of rounds $t \in \left(\frac{T (j-1) \log_2 n }{d}, \frac{T j \log_2 n}{d} \right]$. 
During each phase $j$, the hidden vector and the context vectors are constructed in the same way as  Theorem~\ref{thm:lb-n-eq-expd-strengthened} for dimensions $s \in ((j-1) \cdot \log_2 n, j \cdot \log_2 n]$. 
The entries of the context vectors for the rest of the dimensions (i.e., $s \not\in ((j-1) \cdot \log_2 n, j \cdot \log_2 n]$) are set to $0$. 
}

By our phase construction, $\pi$ can be viewed as a sub-policy in a $\log_2 n$-dimensional space with $n$ arms during phase $j$. One may verify that the length of the $j$-th phase is
\[
\frac{T \log_2 n}{d} \geq \frac{d (\log_2 n)^{1 + \eps} \cdot (\log_2 d)}{d} = (\log_2 n)^{2 + \eps} ,
\] 
satisfying the condition in Theorem~\ref{thm:lb-n-eq-expd-strengthened}. Therefore, by Theorem~\ref{thm:lb-n-eq-expd-strengthened}, the regret of $\pi$ incurred during phase $j$ is  $\Omega \left( \log n \sqrt{\epsilon \frac{T \log n}{d} \ln \frac{T \log n}{d}} \right)$. Therefore, the total regret of the $\frac{d}{\log_2 n}$ phases is at least $\Omega \left( \sqrt{\epsilon dT \log n \ln \frac{T \log n}{d}} \right)$. 
\end{proof}

\medskip
\noindent{\bf Lemma \ref{claim:s-suboptimal-pull-regret} (restated).}
{\it
For any instance $\mathcal{B}^{(U)}$, any coordinate group $s$, and any time $t$, if a policy makes an $s$-suboptimal pull at time $t$, then $r_s^{(t)} \geq {\sqrt{\ln T}}/({36 \sqrt{T}})$.
}
\medskip

\begin{proof}
Let $U = (u_1, u_2, \dots, u_{d/2})$. Assuming that time $t$ is in stage $j$, we verify this claim by calculating the  difference  of two possible expected reward contributions made by the $s$-th coordinate group, as follows.
\begin{multline}\label{eq:s-suboptimal-pull-regret-pre}
\left|\left(z_t \cdot \sqrt{d}, 0\right)^{\top} \left(\theta_{2s - 1}^{(U)}, \theta_{2s}^{(U)}\right) - \left(0, \left(\alpha_{u_s}^{j-1}+ \beta_{u_s}^{j-1}\right) z_t \cdot \sqrt{d}\right)^{\top} \left(\theta_{2s - 1}^{(U)}, \theta_{2s }^{(U)}\right)\right|\\
= z_t \cdot \left| \gamma_{u_s} - \frac{\alpha_{u_s}^{j-1}+ \beta_{u_s}^{j-1}}{2}\right| \geq z_t \cdot \frac{\left|\alpha_{u_s}^{j-1}- \beta_{u_s}^{j-1}\right|}{6}.
\end{multline}
Since $z_t = \sqrt{\frac{S_{t-1} \ln T}{2T}}$ and $\left|\alpha_{u_s}^{j-1}- \beta_{u_s}^{j-1}\right| = 3^{-j}$, we have 
\begin{align*}
\eqref{eq:s-suboptimal-pull-regret-pre} =
\frac{1}{6 \cdot 3^{j}} \sqrt{\frac{S_{t-1} \ln T}{2T}}
\geq \frac{1}{6 \cdot 3^{j}}\sqrt{\frac{S_{t_{j-1}}\ln T}{2T}}
\geq   \frac{\sqrt{\ln T}}{36 \sqrt{T}},
\end{align*}
where the last inequality is because of Eq.~\eqref{eq:stj-lb}. 
\end{proof}

\medskip
\noindent
{{\bf Lemma \ref{lem:subopt-pull-large-in-neighbors} (restated).}
\it
For any stage $j$ and any group $s$, let $U = (u_1, u_2, \dots, u_{d/2})$ and $U' = (u_1', u_2', \dots, u_{d/2}')$ 
be two parameters such that $\tau_{u_s}^{j-1} = \tau_{u_s'}^{j-1}$ but $\tau_{u_s}^{j} \neq \tau_{u_s'}^{j}$, and $u_{a} = u_{a}'$ for every $a \neq s$.
(The definition of $\tau$ can be found in Sec.~\ref{sec:d2-adversarial-construction}.)
Then for any  policy $\pi$ and time period $t$ in stage $j$, it holds that $p_{s,t}^{U,\pi}+p_{s,t}^{U',\pi} \geq 1/2$.
}
\medskip

\begin{proof}
By our construction, we have that 
\[
\left|\theta^{(U)}_{2s-1} - \theta^{(U')}_{2s-1}\right| \leq \frac{1}{\sqrt{d}} \cdot \left(\frac{1}{3}\right)^{j} .
\] 
Therefore, by Claim~\ref{claim:event-prob-difference} (proved below, which is an analogue of Claim~\ref{claim:d2-event-prob-difference}), for any event $E$ at time $t$ before the end of stage $j$,  we have that
\begin{equation}\label{eq:prob-up}
\left| \Pr\left[E | U \right] - \Pr\left[E | U' \right] \right| 
\leq \frac{\sqrt{d}}{2} \left( \sum_{s'=1}^{d/2} \left|\theta^{(U)}_{2s'-1} - \theta^{(U')}_{2s'-1}\right| \right) \sqrt{S_t}
\leq \frac{1}{2} \cdot \left(\frac{1}{3}\right)^{j} \sqrt{S_t} \leq \frac{1}{2}. 
\end{equation}
The last inequality holds because at any time $t$ in stage $j$, it holds that $S_t \leq 9^j$. 

Let $v_1 := (z_t \cdot \sqrt{d}, 0)$ and  $v_2 := (0, (\alpha_{u_s}^{j-1}+ \beta_{u_s}^{j-1}) z_t \cdot \sqrt{d})$. Note that, at any time $t$, the difference between two possible reward contributed by the $s$-th dimension group is 
\[
(v_1 - v_2)^{\top} \left(\theta_{2s-1}^{(U)}, \theta_{2s}^{(U)}\right) = \frac{z_t}{2} \left(2 \gamma_{u_s} - \alpha_{u_s}^{j-1}- \beta_{u_s}^{j-1}\right).
\]
This value is greater than $0$ if and only if $2 \gamma_{u_s} > \alpha_{u_s}^{j-1}+ \beta_{u_s}^{j-1}$. Since $\tau_{u_s}^{j-1} = \tau_{u_s'}^{j-1}$ and $\tau_{u_s}^j \neq  \tau_{u_s'}^j$, by our construction Eq.~\eqref{eq:def-interval}, we have that exactly one of $\gamma_{u_s}$ and $\gamma_{u_s'}$ is greater than $\frac{1}{2}({\alpha_{u_s}^{j-1}+ \beta_{u_s}^{j-1}})$. In other words, at time $t$, any arm that is  $s$-suboptimal  for  parameter is $U$ is not  $s$-suboptimal  for  parameter $U'$, and vice versa. In light of this, let $E$ be the event that at time $t$ policy $\pi$ pulls an arm that is $s$-suboptimal for parameter $U$, and we have that the complement event $\bar{E}$ is that at time $t$ policy $\pi$ pulls an arm that is  $s$-suboptimal for parameter $U'$. By Eq.~\eqref{eq:prob-up}, we have
\begin{align*}
p_{s,t}^{U,\pi}+p_{s,t}^{U',\pi} = \Pr\left[E |U \right] + \Pr\left[\bar{E} | U' \right] = 1 + \Pr\left[E |U \right] - \Pr\left[E | U' \right]  \geq \frac{1}{2} .
\end{align*}
\end{proof}

\begin{claim}\label{claim:event-prob-difference}
For any $U, U'$, let $j$ be the largest number such that $\tau_{u_s}^{j-1} = \tau_{u_s'}^{j-1}$ holds for every $s$. For any time $t \leq t_j$ and any event $E$ that happens at time $t$, we have 
\[
\left| \Pr\left[E | U\right] - \Pr\left[E | U'\right] \right| 
\leq\frac{\sqrt{d}}{2} \left( \sum_{s'=1}^{d/2} \left|\theta^{(U)}_{2s'-1} - \theta^{(U')}_{2s'-1}\right| \right) \sqrt{S_t}.
\]
\end{claim}

\begin{proof}
Note that by our construction, at any time $t\leq t_j$, the contextual vectors of both $\mathcal{B}^{(U)}$ and $\mathcal{B}^{(U')}$ are the same.
Moreover, for any hidden vector and any arm, the reward distribution is a shifted standard Gaussian with variance $1$. 

For any time $t \leq t_j$, let $D_1$ be the product of the arm reward distributions at and before round $t$ when the hidden vector is $\theta^{(U)}$, and let $D_2$ be the same product distribution when the hidden vector is $\theta^{(U')}$. 
Since $\theta^{(U)}_{2s'} = \theta^{(U')}_{2s'}$ for all $s' \in [d/2]$, the difference of the mean rewards at any time $t' : 1 \leq t' \leq t$ for $\theta^{(U)}$ and $\theta^{(U')}$ is at most  $\sum_{s' = 1}^{d/2} \left|\theta^{(U)}_{2s'-1} z_{t'} \sqrt{d} - \theta^{(U')}_{2s'-1} z_{t'} \sqrt{d}\right|$. Note that the KL divergence between two variance-$1$ Gaussians with means $\mu_1$ and $\mu_2$ is $|\mu_1 - \mu_2|^2/2$. Therefore, we have 
\begin{multline*}
\mathrm{KL} \left(D_1 \| D_2\right) \leq   \frac{1}{2} \sum_{i=1}^{t} \left(\sum_{s' = 1}^{d/2} \left|\theta^{(U)}_{2s'-1} z_{t'}  \sqrt{d}- \sum_{s' = 1}^{d/2} \theta^{(U')}_{2s'-1} z_{t'}  \sqrt{d} \right|\right)^2 
= \frac{d}{2} \left( \sum_{s'=1}^{d/2} \left|\theta^{(U)}_{2s'-1} - \theta^{(U')}_{2s'-1}\right| \right)^2 \sum_{i=1}^{t} z^2_{t'} \\
\leq \frac{d}{2} \left( \sum_{s'=1}^{d/2} \left|\theta^{(U)}_{2s'-1} - \theta^{(U')}_{2s'-1}\right| \right)^2 \left(1 + \sum_{i=1}^{t} z^2_{t'} \right)
 = \frac{d}{2} \left( \sum_{s'=1}^{d/2} \left|\theta^{(U)}_{2s'-1} - \theta^{(U')}_{2s'-1}\right| \right)^2 S_t. 
\end{multline*}
Therefore, at time $t$, and for any event $E$,  we have
\[
\left| \Pr\left[E | U \right] - \Pr\left[E | U'\right] \right| 
\leq \sqrt{\frac{1}{2} \mathrm{KL}(D_1 \| D_2)} 
\leq \frac{\sqrt{d}}{2} \left( \sum_{s'=1}^{d/2} \left|\theta^{(U)}_{2s'-1} - \theta^{(U')}_{2s'-1}\right| \right) \sqrt{S_t}
\]
where the first inequality holds because of Pinsker's inequality (Lemma~\ref{lem:pinsker}). 
\end{proof}

{
\subsection{Lower bound for infinite action spaces}\label{sec:infinite-action}

The following corollary establishes an $\Omega(d\sqrt{T\log T})$ lower bound for the case in which
the action spaces for each time period are \emph{infinite} and \emph{changing} over time.
\begin{corollary}
For any policy $\pi$, there exists a bandit instance with regression model $\theta\in\mathbb R^d$, $\|\theta\|_2\leq 1$
and closed action spaces $\mathcal A_1,\cdots,\mathcal A_T\subseteq\{x\in\mathbb R^d:\|x\|_2\leq 1\}$,
$|\mathcal A_t|=\infty$, such that for sufficiently large $T$, 
$$
\mathbb E[R^T] = \Omega(d\sqrt{T\log T}).
$$
\label{cor:infinite-action}
\end{corollary}
\begin{proof}
We prove the corollary by contradiction. Suppose the opposite, that there exists a policy $\pi$ such that for all $\theta \in \mathbb R^d, \|\theta\|_2 \leq 1$ and measurable action spaces $\mathcal A_1,\cdots,\mathcal A_T\subseteq\{x\in\mathbb R^d:\|x\|_2\leq 1\}$ such that $|\mathcal A_t| = \infty$, it holds that $\mathbb E[R^T] = o(d \sqrt{T \log T})$.  We will show that there exists a policy $\pi'$ to achieve $o(d \sqrt{T \log T})$ on all bandit instances with $\theta \in \mathbb R^d, \|\theta\|_2 \leq 1$, and action sets $\mathcal A_1',\cdots,\mathcal A_T' \subseteq\{y\in\mathbb R^d:\|y\|_2\leq 1, y^\top \theta \geq 0\}$ such that $|\mathcal A_t' | = 2^{d/2}$, contradicting Theorem~\ref{thm:LB-n-eq-expd} and Theorem~\ref{thm:lb-n-eq-expd-strengthened}.

For action sets $\mathcal A_1',\cdots,\mathcal A_T' \subseteq\{y\in\mathbb R^d:\|y\|_2\leq 1,  y^\top \theta \geq 0\}$ such that $|\mathcal A_t' | = 2^{d/2}$, the policy $\pi'$ construct $\mathcal A_t = \{x{(\lambda, y)} = \lambda y | y \in \mathcal A_t', \lambda\in [0, 1]\}$ and simulate policy $\pi$ when the candidate action space is $\mathcal A_t$. Clearly, $\mathcal A_t$ is closed and $|\mathcal A_t| = \infty$. If policy $\pi$ decides to play $x_t = x{(\lambda, y)} \in \mathcal A_t$, the policy $\pi'$ plays $y_t = y \in \mathcal A_t'$, observes the reward $r_t$, and feeds $\lambda r_t$ as reward to the policy $\pi$. Since $r_t - y_t^\top \theta$ is a centered sub-Gaussian with variance proxy 1, we have that $\lambda r_t - \lambda x_t^\top \theta$ is also a centered sub-Gaussian with variance proxy (at most) 1.

Since $y^\top \theta \geq 0$ for all $y \in \mathcal A_t'$ and all $t$, we upper bound the expected regret incurred by $\pi'$ as follows.
\begin{multline*}
\mathbb E \left[R^T_{\pi'}\right] = \sum_{t=1}^{T} \mathbb E \left[\max_{y \in \mathcal A_t'} y^\top \theta - y_t^\top \theta \right] =  \sum_{t=1}^{T} \mathbb E \left[\sup_{x \in \mathcal A_t} x^\top \theta - y_t^\top \theta \right]\\
 \leq  \sum_{t=1}^{T} \mathbb E \left[\sup_{x \in \mathcal A_t} x^\top \theta - x_t^\top \theta \right] = \mathbb E\left[R^T_{\pi}\right] \leq o(d\sqrt{T \log T}),
\end{multline*}
where we use the subscript in $R^T$ to denote whether the regret is incurred by the policy $\pi$ or $\pi'$, and the first inequality is because, if at time $t$, $x_t = x(\lambda, y_t)$ is chosen by policy $\pi$, then $y_t^\top\theta = x_t^\top\theta/ \lambda \geq x_t^\top\theta$.
\end{proof}

}

\appendix

\section{Probability tools}

The following lemma is the Hoeffding's concentration inequality for sub-Gaussian random variables,
which can be found in for example \citep{hoeffding1963probability}.

\begin{lemma}
Let $X_1,\cdots,X_n$ be independent centered sub-Gaussian random variables with sub-Gaussian parameter $\sigma^2$. Then for any $\xi>0$,
$$
\Pr\left[\left|\sum_{i=1}^n X_i \right|\geq \xi\right] \leq 2 \exp\left\{-\frac{\xi^2}{2n\sigma^2}\right\}.
$$
\label{lem:subgaussian-concentration}
\end{lemma}

The following lemma states Pinsker's inequality \citep{pinsker1960information}.

\begin{lemma}\label{lem:pinsker}
If $P$ and $Q$ are two probability distributions on a measurable space $(X, \Sigma)$, then for any measurable event $A \in \Sigma$, it holds that
\[
\left| P(A) - Q(A) \right| \leq \sqrt{\frac{1}{2} \mathrm{KL}(P \| Q)},
\]
where 
\[ 
\mathrm{KL}(P \| Q) = \int_X \left(\ln \frac{\mathrm d P}{\mathrm d Q}\right) \mathrm d P
\]
is the Kullback--Leibler divergence.
\end{lemma}

\section{Additional proofs in Sec.~\ref{sec:ub}} \label{app:proofs-ub}

Below we provide the proof of Lemma~\ref{lem:expected-tail}.

\medskip
\noindent
{{\bf Lemma \ref{lem:expected-tail} (restated).}
\it
For any $\zeta,t$ and $i$, we have that 
$$
\mathbb E\big[\vct 1\{\neg\mathcal E_{\zeta,t}^i\}\cdot\big|x_{it}^\top(\hat\theta_{\zeta,t}-\theta)\big||\zeta< \zeta_t,i\in\mathcal N_{\zeta,t}\big] \leq \sqrt{2 \pi d}/(n\zeta_0\sqrt{T}).
$$
}

\begin{proof}
 The entire proof is carried out conditioned on 
 \begin{align*}
& \{\mathcal X_{\zeta',t-1}\}_{\zeta'\leq\zeta},\\
 &\{i_{t'}\}_{t' \in \mathcal X_{\zeta, t-1} \cup \mathcal X_{\zeta-1, t-1} \cup \dots \cup \mathcal X_{0, t-1}},\\
 &\Lambda_{\zeta,t-1}, \Lambda_{\zeta-1,t-1}, \Lambda_{\zeta-2,t-1}, \dots \Lambda_{0,t-1},\\
 &\lambda_{\zeta-1,t-1}, \lambda_{\zeta-2,t-1}, \lambda_{\zeta-3,t-1}, \dots \lambda_{0,t-1}.
 \end{align*}
 This renders the quantities of $\alpha_{\zeta,t}^i,\omega_{\zeta,t}^i$ deterministic. Note that the event $\zeta < \zeta_t$ and the set $\mathcal{N}_{\zeta, t}$ also become deterministic. Therefore, we only need to prove that $\mathbb E[\vct 1\{\neg\mathcal E_{\zeta,t}^i\}\cdot |x_{it}^\top(\hat\theta_{\zeta,t}-\theta)|] \leq \sqrt{2 \pi d}/(n\zeta_0\sqrt{T})$  assuming that $\zeta < \zeta_t$ and $i\in\mathcal N_{\zeta,t}$.
 
 Note also that, by Proposition~\ref{prop:layer-independence},
the quantities $\{\varepsilon_{t'}\}_{t'\in\mathcal X_{\zeta,t}}$ remain independent, centered sub-Gaussian random variables.
 
 We first derive an upper bound on the tail of $|x_{it}^\top(\hat\theta_{\zeta,t}-\theta)|$.
 Note that $\mathcal X_{\zeta,t-1}$ is the set of all time periods $\tau<t$ such that $\zeta_\tau=\zeta$.
 By elementary algebra, we have
 \begin{align*}
 x_{it}^\top(\theta - \hat\theta_{\zeta,t})
 &= x_{it}^\top (\theta - \Lambda_{\zeta,t-1}^{-1} \lambda_{\zeta,t-1})
 = x_{it}^\top \left(\theta - \Lambda_{\zeta,t-1}^{-1} \sum_{t' \in \mathcal X_{\zeta,t-1}} x_{i_{t'},t'} (x^\top_{i_{t'},t'} \theta + \eps_{t'}) \right) \\
 &= x_{it}^\top \left(\theta - \Lambda_{\zeta,t-1}^{-1} (\Lambda_{\zeta,t-1} - I) \theta - \Lambda_{\zeta,t-1}^{-1} \sum_{t' \in \mathcal X_{\zeta,t-1}} x_{i_{t'},t'} \eps_{t'} \right) \\
 &= x_{it}^\top \Lambda_{\zeta,t-1}^{-1} \left(\theta - \sum_{t' \in \mathcal X_{\zeta,t-1}} x_{i_{t'},t'} \eps_{t'} \right). 
 \end{align*}
 Subsequently, 
 \begin{equation}\label{eq_bound_sub}
 \left| (\theta - \hat{\theta}_t)^\top x_{it} \right| 
 \leq \left|x_{it}^\top \Lambda_{\zeta,t-1}^{-1} \theta \right| 
 + \left| \sum_{t' \in \mathcal X_{\zeta,t-1}} x_{it}^\top \Lambda_{\zeta,t-1}^{-1} x_{i_{t'},t'} \eps_{t'} \right| .
 \end{equation}
 
 For the first term in the RHS (right-hand side) of Eq.~\eqref{eq_bound_sub}, applying Cauchy-Schwarz inequality and the facts that $\Lambda_{\zeta,t-1}\succeq I$, 
 $\|\theta|_2\leq 1$ we have $|x_{it}^\top\Lambda_{\zeta,t-1}^{-1}\theta| \leq \|\Lambda_{\zeta,t-1}^{-1/2}x_{it}\|_2\|\Lambda_{\zeta,t-1}^{-1/2}\|_{\mathrm{op}}\|\theta\|_2
 \leq \sqrt{x_{it}^\top\Lambda_{\zeta,t-1}^{-1}x_{it}}$.
 For the second term in the RHS of Eq.~\eqref{eq_bound_sub}, 
 because $\{\varepsilon_{t'}\}$ are centered sub-Gaussian variables with sub-Gaussian parameter $1$
 and $\{\varepsilon_{t'}\}_{t'\in\mathcal X_{\zeta,t-1}}$ are still statistically independent even after the conditioning (By Proposition~\ref{prop:layer-independence}),
 we conclude that $ \sum_{t' \in \mathcal X_{\zeta,t-1}} x_{it}^\top \Lambda_{\zeta,t-1}^{-1} x_{i_{t'},t'} \varepsilon_{t'}$ is also a centered sub-Gaussian random variable with sub-Gaussian
 parameter upper bounded by
 \begin{equation} \label{eq:bound_sub-2}
 \sum_{t' \in \mathcal X_{\zeta,t-1}} (x_{it}^\top \Lambda_{\zeta,t-1}^{-1} x_{i_{t'},t'})^2
 = x_{it}^\top \Lambda_{\zeta,t-1}^{-1} \left(\sum_{t' \in \mathcal X_{\zeta,t-1}} x_{i_{t'},t'} x^\top_{i_{t'},t'} \right) \Lambda_{\zeta,t-1}^{-1} x_{it} = 
 x_{it}^\top\Lambda_{\zeta,t-1}^{-1}x_{it}.
 \end{equation} 
 Combining Eqs.~\eqref{eq_bound_sub}, \eqref{eq:bound_sub-2}, and using standard concentration inequalities of sub-Gaussian random variables (see for example Lemma \ref{lem:subgaussian-concentration}), we have for every $\delta\in(0,1)$ that 
 \begin{eqnarray}
 \Pr \left[(\theta - \hat\theta_{\zeta,t-1})^\top x
 \geq \left(\sqrt{2 \ln ( \delta^{-1})} + 1 \right)\sqrt{x_{it}^\top\Lambda_{\zeta,t-1}^{-1}x_{it}} \right] 
 \leq \delta,
 \label{eq:tail-prediction-error-1}
 \end{eqnarray}
 which is equivalent to
  \begin{eqnarray}
 \Pr \left[(\theta - \hat\theta_{\zeta,t-1})^\top x
 \geq\beta\omega_{\zeta,t}^i \right]  \leq e^{-(\beta-1)^2/2}, \;\;\;\;\;\;\forall \beta\geq 1.
 \label{eq:tail-prediction-error-2}
 \end{eqnarray}
 
 Integrating both sides of Eq.~(\ref{eq:tail-prediction-error-2}) from $\alpha_{\zeta,t}^i\omega_{\zeta,t}^i$ to $+\infty$ we obtain
 \begin{align}
 \mathbb E[\vct 1\{\neg\mathcal E_{\zeta,t}^i\}\cdot |x_{it}^\top(\hat\theta_{\zeta,t}-\theta)|] 
 &\leq \int_{\alpha_{\zeta,t}^i\omega_{\zeta,t}^i}^{+\infty}
 \Pr\left[|x_{it}^\top(\hat\theta_{\zeta,t}-\theta)|\geq u\right]\ud u\nonumber\\
 &\leq  \int_{\alpha_{\zeta,t}^i\omega_{\zeta,t}^i}^{+\infty} e^{-(u/\omega_{\zeta,t}^i-1)^2/2}\ud u
 = \sqrt{2 \pi}\omega_{\zeta,t}^i \int_{\alpha_{\zeta,t}^i-1}^{+\infty} \frac{1}{\sqrt{2} \pi}e^{-v^2/2}\ud v\nonumber\\
 &\leq \sqrt{2 \pi}\omega_{\zeta,t}^i \cdot e^{-(\alpha_{\zeta,t}^i-1)^2/2},\label{eq:tail-prediction-error-3}
 \end{align}
 where the last inequality again holds by tail bounds of Gaussian random variables (Lemma \ref{lem:subgaussian-concentration}).
 Plugging in the expression of $\alpha_{\zeta,t}^i$ in Algorithm \ref{alg:suplinucb}, the right-hand side of Eq.~(\ref{eq:tail-prediction-error-3}) can be upper bounded by
 $$
 \sqrt{2 \pi}\omega_{\zeta,t}^i\cdot \exp\left\{-\frac{\max\{1,\ln(T(\omega_{\zeta,t}^i)^2/d)\}\ln(n^2\zeta_0^2)}{2}\right\}
 \leq \sqrt{2 \pi}\omega_{\zeta,t}^i \sqrt{\frac{d}{T(\omega_{\zeta,t}^i)^2}\frac{1}{n^2\zeta_0^2}} \leq \frac{\sqrt{2 \pi d}}{n\zeta_0\sqrt{T}},
 $$
 which is to be demonstrated. 
\end{proof}

\bibliographystyle{apa-good}
\bibliography{refs-yining}





\end{document}